%% file: neurips_2021.tex
\documentclass{article}

\usepackage[nonatbib, final]{neurips_2021}

\usepackage{cancel}
\usepackage{amsmath}
\usepackage{amssymb}
\usepackage[utf8]{inputenc} 
\usepackage[T1]{fontenc}    
\usepackage{hyperref}       
\usepackage{url}            
\usepackage{booktabs}       
\usepackage{amsfonts}       
\usepackage{nicefrac}       
\usepackage{microtype}      
\usepackage[dvipsnames]{xcolor}         
\usepackage{bm}
\usepackage{tikz}
\usepackage{sistyle}
\SIthousandsep{,}
\usepackage{arabtex}
\renewenvironment{abstract}%
{%
  \vskip 0.075in%
  \centerline%
  {\large\bf Abstract}%
  \vspace{0.5ex}%
  \begin{quote}%
}
{
  \par%
  \end{quote}%
  \vskip 1ex%
}

\usepackage{algorithmic}
\usepackage[linesnumbered,ruled,vlined]{algorithm2e}
\usepackage{tikz}
\usepackage{todonotes} 
\usetikzlibrary{matrix}
\usetikzlibrary{positioning}
\usepackage{tablefootnote}
\usepackage{booktabs} 
\usepackage{color, colortbl} 
\usepackage{subcaption}
\usepackage{pgfplots}
\usepgfplotslibrary{groupplots}

\SetKwInput{KwInput}{Input}                
\SetKwInput{KwOutput}{Output}              
\usepackage{amsthm}

\newtheorem{theorem}{Theorem}

\newtheorem{proposition}{Proposition}

\usepackage{multirow}

\renewcommand{\l}[0]{l} 
\newcommand{\m}[0]{m} 
\newcommand{\f}[0]{k} 

\newcommand{\one}[1]{g} 
\newcommand{\T}[0]{T} 
\newcommand{\p}[0]{B} 
\newcommand{\pred}[0]{g} 
\newcommand{\score}{\mathsf{cost}}
\newcommand{\ascore}{\mathsf{cost}^a}
\newcommand{\imp}{\mathsf{imp}}
\newcommand{\aimp}{\mathsf{imp}^a}
\newcommand{\prob}{\mathsf{pr}}
\newcommand{\aprob}{\mathsf{pr}^a}
\newcommand{\afilter}{\p(\T)_{\phi}}
\newcommand{\afilterneg}{\p(\T)_{\neg\phi}}
\newcommand{\size}{\mathsf{size}}
\newcommand{\bestsplit}{\mathsf{split}}
\newcommand{\abestsplit}{\mathsf{split}^a}

\newcommand{\toy}{\T^\mathsf{ex}}
\newcommand{\NAME}{Antidote-P}

    \usetikzlibrary{
        pgfplots.colorbrewer,
    }
    \pgfplotsset{
        cycle list/.define={my marks}{
            every mark/.append style={dotted,fill=\pgfkeysvalueof{/pgfplots/mark list fill}},mark=none\\
            every mark/.append style={solid,fill=\pgfkeysvalueof{/pgfplots/mark list fill}},mark=square*\\
            every mark/.append style={solid,fill=\pgfkeysvalueof{/pgfplots/mark list fill}},mark=none\\
            every mark/.append style={solid,fill=\pgfkeysvalueof{/pgfplots/mark list fill}},mark=diamond*\\
        },
    }

\DeclareMathOperator*{\argmax}{argmax}
\DeclareMathOperator*{\argmin}{argmin}
\usepackage[capitalize]{cleveref}

\crefname{section}{\S}{section}

\usepackage[backend=bibtex]{biblatex}
\title{Certifying Robustness to Programmable Data Bias in Decision Trees}

\author{%
  Anna P. Meyer, Aws Albarghouthi\thanks{~Author's name in native alphabet: \novocalize\RL{'aws albr.gU_ty}}, and Loris D'Antoni\\
  Department of Computer Sciences\\
  University of Wisconsin--Madison\\
  Madison, WI 53706 \\
  \texttt{\{annameyer}, \texttt{aws}, \texttt{loris\}@cs.wisc.edu} \\
}

\renewcommand{\paragraph}[1]{\textbf{#1.}}
\newtheorem{example}{Example}[section]
\renewcommand{\leq}{\leqslant}
\pgfplotsset{compat=1.17}

\begin{document}

\maketitle

\begin{abstract}
Datasets can be biased due to societal inequities, human biases, under-representation of minorities, etc.
Our goal is to \emph{certify}
that models produced by a learning algorithm are \emph{pointwise-robust} to potential dataset biases.
This is a challenging problem: it entails learning models for a large, or even infinite, number of datasets, ensuring that they all produce the same prediction.
We focus on decision-tree learning due to the interpretable nature of the models.
Our approach allows programmatically specifying \emph{bias models} across a variety of dimensions (e.g., missing data for minorities), composing types of bias, and targeting bias towards a specific group.
To certify robustness, we use a novel symbolic technique to evaluate a decision-tree learner on a large, or infinite, number of datasets, certifying that each and every dataset produces the same prediction for a specific test point.
We evaluate our approach on datasets that are commonly used in the fairness literature,
and demonstrate our approach's viability on a range of bias models.
\end{abstract}


\input{sections/intro.tex} 
\input{sections/related_work.tex}

\input{sections/bias_models.tex}

\input{sections/certification.tex}
\input{sections/experiments.tex}

\input{sections/conclusions.tex}

\begin{ack}
We thank the anonymous reviewers for commenting on earlier drafts and Sam Drews for the generous use of his code. 
This work is supported by the National Science Foundation 
grants CCF-1420866, CCF-1704117, CCF-1750965, CCF-1763871, CCF-1918211, CCF-1652140, 
a Microsoft Faculty Fellowship, and gifts and awards from Facebook and Amazon.
\end{ack}

\renewcommand*{\bibfont}{\small}
\printbibliography


\newpage
\appendix

\input{appendix/abstraction_details.tex}
\input{appendix/composition.tex}

\input{appendix/soundness.tex}
\input{appendix/precision.tex}

\input{appendix/experiments.tex}

\end{document}

%% file: sections/intro.tex
\section{Introduction}
The proliferation of machine-learning algorithms
has raised alarming questions about fairness in automated decision-making~\cite{barocas-hardt-narayanan}.
In this paper, we focus our attention on bias in training data.
Data can be biased due to societal inequities, human biases, under-representation of minorities,
malicious data \emph{poisoning}, etc.
For instance, historical data can contain human biases, e.g., certain individuals' loan requests get rejected, although (if discrimination were not present) they should have been approved, or
  women in certain departments are consistently given lower performance scores by managers.
 
 Given biased training data, 
 we are often unable to de-bias it because we do not know which samples are affected.
 This paper asks,
 \emph{can we certify (prove) that our predictions are robust under a given form and degree of bias in the training data?} We aim to answer this question without having to show
 which data are biased (i.e., poisoned).
 %
 Techniques for certifying poisoning robustness 
  (\emph{i}) focus on specific poisoning forms, e.g., label-flipping~\cite{rosenfeld-randomized}, 
 or (\emph{ii}) perform certification using defenses that create
 complex, uninterpretable classifiers, e.g., due to randomization or ensembling~\cite{jia2021intrinsic,levine2020deep,rosenfeld-randomized}.
 To address limitation (\emph{i}), we present  \emph{programmable bias definitions} that model nuanced biases in practical domains.
 To address (\emph{ii}),
we target \emph{existing} decision-tree learners---%
 considered interpretable and desirable for
 sensitive decision-making~\cite{Rudin2019Why}---and exactly certify their robustness,
 i.e., provide proofs that the bias in the data will not affect the outcome of the trained
 model on a given point.%
 

 We begin by presenting a \emph{language} for programmatically defining \emph{bias models}.
 A bias model allows us to flexibly  specify what sort of bias we suspect to be in the data, e.g., up to $n\%$ of the women \emph{may have} wrongly received a negative job evaluation.
 Our bias-model language is generic, allowing us to \emph{compose} simpler bias models into more complex ones, e.g.,  up to $n\%$ of the women may have wrongly received a negative evaluation \emph{and} up to $m\%$ of Black men's records may have been completely missed.
The choice of bias model depends on the provenance of the data and the task.

After specifying a bias model, our goal is to certify \emph{pointwise robustness to data bias}: Given an input $x$, we want to ensure that no matter whether the training data is biased or not, the resulting model's prediction for $x$ remains the same.
Certifying pointwise robustness is challenging.
One can train a model for every perturbation (as per a bias model) of a dataset
and make sure they all agree.
But this is generally not feasible, because the set of possible perturbations can be large or infinite.
Recall the bias model where up to $n\%$ of women may have wrongly received a negative label.
For a dataset with \num{1000} women and $n=1\%$, there are more than $10^{23}$ possible perturbed datasets.

To perform bias-robustness certification on decision-tree learners,
we employ \emph{abstract interpretation}~\cite{cousot1977abstract}
to symbolically run the decision-tree-learning algorithm on a large or infinite set of datasets simultaneously, thus learning a \emph{set} of possible decision trees, represented compactly.
The crux of our approach is a technique that lifts operations of decision-tree learning 
to symbolically operate over a \emph{set of datasets} defined using our bias-model language.
As a starting point, we build upon Drews et al.'s~\cite{drews-pldi} demonstration of poisoning-robustness certification for the simple bias model where an adversary may have added fake training data.
Our approach completely reworks and extends their technique to target the bias-robustness problem and handle complex bias models, including ones that may result in an infinite number of datasets.

\begin{figure}
\includegraphics[width=\textwidth]{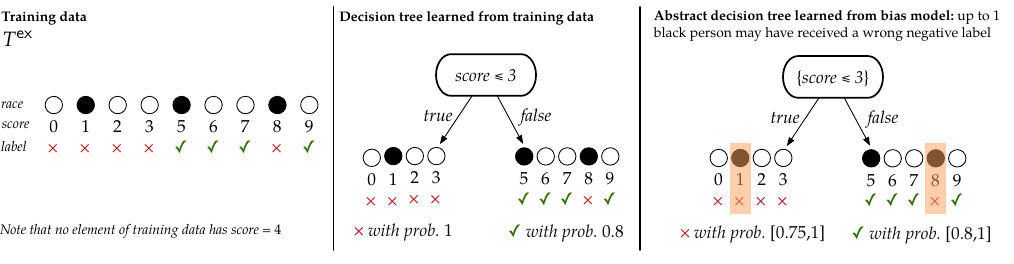}
\caption{A simple, hypothetical running example}\label{fig:running}
\end{figure}

\paragraph{Contributions}
We make three contributions:
(1) We formalize the bias-robustness-certification problem
and present a language to compositionally define bias models.
(2) We present a symbolic technique that performs decision-tree learning on
a set of datasets defined by a bias model, allowing us to perform certification.
(3) We evaluate our approach on a number of bias models and datasets from
the fairness literature. Our tool can certify pointwise robustness for a variety of bias models; we also show that some datasets have \emph{unequal robustness-certification rates} across demographics groups.

\paragraph{Running example}
Consider the example in \cref{fig:running};
our goal is to classify who should be hired based on a test score.
A standard decision-tree-learning algorithm  would choose the split (predicate) $\emph{score} \leq 3$, assuming we restrict tree depth to 1.\footnote{Other predicates, e.g., $\textit{score}\leq 4$, will yield the same split. We choose a single split for illustrative purposes here. (The implementation considers all possible splits that yield distinct partitions, so it would consider $\textit{score}\leq 3$ and $\textit{score}\leq 4$ as a single entity.)}
As shown in \cref{fig:running} (middle), the classification depends 
on the data split; e.g., on the right hand side, we see that a person with $\emph{score} > 3$
is accepted, because the proportion (``probability'') of the data with positive labels and $\emph{score} > 3$ is $4/5$ ($> 1/2$).

Now suppose that our bias model says that up to one Black person in the dataset
may have received a wrongful rejection.
Our goal is to show that even if that is the case, the prediction of a new test sample $x$ will not change.
As described above, training decision trees for all possible modified datasets is generally intractable.
Instead, we symbolically learn a set of possible decision trees compactly, as illustrated in \cref{fig:running} (right).
In this case the learning algorithm
always chooses $\emph{score} \leq 3$ (generally, our algorithm can capture all viable splits).
However, the proportion of labels on either branch varies.
For example, on the right, if the highlighted sample is wrongly labeled, then the ratio changes from $0.8$ to $1$.
To efficiently perform this calculation,
we lift the learning algorithm's operations to  \emph{interval arithmetic} and represent the probability as $[0.8,1]$.
Given a new test sample $x=\langle$\emph{race}=\textrm{Black}, \emph{score}=7$\rangle$, we follow the right branch and, since the interval is always larger than 0.5, 
we certify that the algorithm is robust for $x$. 
In general, however, due to the use of abstraction, our approach may fail to find tight intervals, and therefore be unable to certify robustness for all robust inputs.




%% file: sections/related_work.tex
\section{Related work}
\paragraph{Ties to poisoning}
Our  dataset bias language captures existing definitions
of \emph{data poisoning}, where an attacker is assumed to have maliciously modified training data.
Poisoning has been studied extensively. Most works have focused on attacks~\cite{biggio2012poisoning,cheng2018queryefficient, li_dt_instability, shafahi2018poison, turney-instability,xiao-label-flipping,zhang-label-flipping} or on training models that are empirically less vulnerable (defenses)~\cite{andriushchenko-stumps,chen-robust,dwyer_dt_instability,paudice-label-sanitation,rosenfeld-randomized,steinhardt-certified}. 
Our work differs along a number of dimensions:
(1) We allow programmatic, custom, composable definitions of bias models; notably, 
to our knowledge, no other work in this space allows for \emph{targeted} bias, i.e., restricting bias to a particular subgroup.
(2) Our work aims to certify and quantify robustness of an existing decision-tree algorithm, not to modify it (e.g., via bagging or randomized smoothing) to improve robustness~\cite{jia2021intrinsic,levine2020deep,rosenfeld-randomized}.

Statistical defenses show that a learner is robust \emph{with high probability}, often by modifying a base learner using, e.g., randomized smoothing~\cite{rosenfeld-randomized}, outlier detection~\cite{steinhardt-certified}, or bagging~\cite{jia2021certified,jia2021intrinsic}. 
Non-statistical certification (including abstract interpretation) 
has mainly focused on \emph{test-time} robustness,
where the vicinity (e.g., within an $\ell_p$ norm) of an input is proved to receive the same prediction~\cite{anderson-neuralnets-robustness, gehr-neuralNet-robustness, ranzato_abstract, singh-neuralNet-robustness, tornblom-abstract, albarghouthi-book}.
Test-time robustness is a simpler problem than our \emph{train-time} robustness problem because
it does not have to consider the mechanics of the learner on sets of datasets.
The only work we know of that certifies train-time robustness of decision trees is by Drews et al.~\cite{drews-pldi} and focuses on poisoning attacks where an adversary \emph{adds fake data}.
Our work makes a number of significant leaps beyond this work:
(1) We frame data bias as programmable, rather than fixed, 
to mimic real-world bias scenarios, an idea that has
gained traction in a variety of domains, e.g., NLP~\cite{yuhao}.
(2) We lift a decision-tree-learning algorithm to operate over sets of datasets
represented via our bias-model language.
(3) We investigate the bias-robustness problem through a fairness lens, particularly with an eye towards robustness rates for various demographic groups.

\paragraph{Ties to fairness}
The notion of \emph{individual fairness} specifies that \emph{similar individuals} should receive similar predictions~\cite{dwork-fairness}; by contrast, we certify that no individual should receive different predictions under models trained by \emph{similar datasets}.
Black and Fredrickson explore the problem of how individuals' predictions change under models trained by similar datasets, but their concept of similarity is limited to removing a \emph{single} data point~\cite{black-leaveoneout}.
Data bias, in particular, has received some attention in the fairness literature. Chen et al. suggest adding missing data as an effective approach to remedying bias in machine learning~\cite{chen-irene}, which is one operation that our bias language captures.
Mandal et al. build on the field of \emph{distributional robustness}~\cite{ben-tal-distribution,namkoong-distribution, shafieezadeh-distribution} to build classifiers that are empirically group-fair across a variety of nearby distributions~\cite{mandal-fairness-checking}. Our problem domain is related to distributional robustness because we certify robustness over a family of similar datasets; however, we define specific data-transformation operators to define similarity, and, unlike Mandal et al., we certify existing learners instead of building empirically robust models.

\paragraph{Ties to robust statistics}
There has been renewed interest in robust statistics for machine learning~\cite{diakonikolas2019recent,DBLP:journals/cacm/DiakonikolasKKL21}.
Much of the work concerns outlier detection
for various learning settings, e.g., estimating parameters of a Gaussian.
The distinctions are two-fold: (1) We deal with rich, nuanced bias models, as opposed to out-of-distribution samples, and
(2) we aim to certify that predictions are robust for a specific input,
a guarantee that cannot be made by robust-statistics-based techniques~\cite{diakonikolas2019sever}.

%% file: sections/bias_models.tex
\section{Defining data bias programmatically}\label{sec:3}
We define 
the \emph{bias-robustness problem}
and a language for defining bias models \emph{programmatically}.

\newcommand{\calx}{\mathcal{X}}
\newcommand{\caly}{\mathcal{Y}}
\newcommand{\calh}{\mathcal{H}}
\newcommand{\h}{h}
\newcommand{\learn}{A}
\newcommand{\miss}{\textsc{miss}}
\newcommand{\missB}[2]{\miss_{#1}^{#2}}
\newcommand{\flip}{\textsc{flip}}
\newcommand{\flipB}[2]{\flip_{#1}^{#2}}
\newcommand{\fake}{\textsc{fake}}
\newcommand{\fakeB}[2]{\fake_{#1}^{#2}}
\newcommand{\infer}{\mathsf{infer}}
\newcommand{\ainfer}{\mathsf{infer}^a}

\paragraph{Bias models}
A dataset $\T \subseteq \calx \times \caly$
is a set of pairs of samples and labels,
where $\caly = \{0,\ldots,n-1\}$.
For a dataset $\T$, we will use $T_\calx$ to denote $\{x \mid (x,y) \in \T\}$.
A \emph{bias model} $\p$ is a function that takes a dataset 
and returns a \emph{set of datasets}.
We call $\p(\T) \subseteq 2^{\calx \times \caly}$ a \emph{bias set}.
We assume that $\T \in \p(\T)$.
Intuitively, $\p(\T)$ represents \emph{all datasets that could have existed had there been no bias.}


\paragraph{Pointwise data-bias robustness}
Assume we have a learning algorithm $\learn$
that, given a training dataset $\T$,
 deterministically returns a classifier $\h_\T$ from some hypothesis class $\calh$.
Fix a dataset $\T$ and bias model $\p$.
Given a sample $x \in \calx$, we say that $\learn$ is \emph{pointwise robust} (or robust for short) on $x$  iff 
\begin{equation}\label{eq:rob}
\text{there is a label $i$ such 
that for all $\T' \in \p(\T)$, we have $\h_{\T'}(x) = i$}
\end{equation}


\paragraph{Basic components of a bias model}
We begin with basic bias models.
%

 \emph{Missing data}: 
A common bias in datasets is missing data, which can occur via poor historical representation of a subgroup (e.g., women in CS-department admissions data), or from present-day biases or shortsightedness (e.g., a survey that bypasses low-income neighborhoods).
Using the parameter $\m$ as the maximum number of missing elements, we formally define:
\begin{equation*}
        \missB{\m}{}(\T) = \{\T'  \mid \T' {\supseteq} \T,\ |\T' {\setminus} \T| \leq \m\}
\end{equation*}
    
$\missB{\m}{}(\T)$ defines an \emph{infinite} number
datasets when the sample space is infinite (e.g., $\mathbb{R}$-valued).
  
\begin{example}
Using $\toy$ from \cref{fig:running}, 
$\miss_1(\toy)$
is the set of all datasets that are either $\toy$
or $\toy$~plus any new element $(x,y)$
with arbitrary race, score, and label.
\end{example}

 \emph{Label flipping}: 
 Historical data can contain human biases, e.g., in loan financing, certain individuals' loan requests get rejected due to discrimination.
 Or consider employee-performance data, where women in certain departments are consistently given lower scores by managers.
 We model such biases as label flipping, where labels of up to $\l$
 individuals in the dataset may be incorrect:
 \begin{equation*}
     \flipB{\l}{}(\T) = \{\T'  \mid |\T|{=}|\T'|,\ |\T {\setminus} \T'| \leq \l,\ \T_\calx = \T'_\calx \}  
 \end{equation*}

\begin{example}
Using $\toy$ from Fig. \ref{fig:running} and a bias model $\flip_1$,
we have
$\flip_1(\toy) = \{\toy_{0},\ldots,\toy_3,$ $ \toy_5,\ldots,\toy_9,\toy\}$,
where $\toy_i$ is $\toy$ with the label of the element with $\emph{score}=i$ \emph{changed}.
\end{example}

 \emph{Fake data}: 
    Our final bias model assumes the dataset may contain \emph{fake} data. One cause may be a malicious user who enters fraudulent data into a system (often referred to as \emph{poisoning}). Alternatively, this model can be thought of as the inverse of $\miss$, e.g., we over-collected data for men.
    $$\fakeB{\f}{}(\T) = \{\T'  \mid \T' {\subseteq} \T,\ |\T {\setminus} \T'| \leq \f\}$$

\begin{example}
Using $\toy$ from \cref{fig:running} and a bias model $\fakeB{1}{}$, we get $\fakeB{1}{}(\toy) = \{\toy_{\downarrow{0}},\ldots, \toy_{\downarrow{3}},$ $\toy_{\downarrow 5},\ldots, \toy_{\downarrow 9}, \toy \}$
where $\toy_{\downarrow i}$ is $\toy$ such that the element with $\emph{score}=i$ has been
\emph{removed}.
\end{example}


\paragraph{Targeted bias models}
Each bias model has a \emph{targeted} version that limits the bias to a specified group of data points. For example, consider the missing data transformation. If we suspect that data about women is missing from an HR database, we can limit the $\miss$ transformation to only add data points with $\emph{gender}=\textrm{female}$. 
Formally, we define a predicate $\pred:\calx \times \caly \to \mathbb{B}$,
where $\mathbb{B}=\{\emph{true},\ \emph{false}\}$.
%
%
$$\missB{\m}{\pred}(\T) = \{\T'  \mid \T' \supseteq \T,\ |\T' \setminus \T| \leq \m, \ \textrm{and }g(x,y)\textrm{ is true }\forall (x,y)\in (T'\setminus T)\}$$ 
Targeted versions of label-flipping and fake data can be defined in a similar way.

\begin{example}
In \cref{fig:running} (right),
we used bias model $\flip_1^g$, 
where $g$ targets Black people with negative labels.
This results in the bias set 
$\flip_1^g(\toy) = \{\toy_1, \toy_8, \toy\}$,
where $\toy_i$ is $\toy$ with the label of the element with $\mathit{score}=i$
\emph{changed} (recall that scores 1 and 8 belong to Black people in $\toy$).
\end{example}

\paragraph{Composite bias models}
We can compose basic components to generate a \emph{composite model}. Specifically, we define a composite model $\p$ as a finite set of arbitrary basic components, that is, 
\begin{align}
    \p = [ \missB{\m_1}{\pred_1},\ldots,\missB{\m_j}{\pred_j}, \flipB{\l_1}{\pred_{j+1}},\ldots,\flipB{\l_p}{\pred_{j+p}}, \fakeB{\f_1}{\pred_{j+p+1}},\ldots,\fakeB{\f_q}{\pred_{j+p+q}} ]
\end{align}
 $\p(\T)$ is generated from $\T$ by applying the basic components of $\p$ iteratively.
We must apply the constituent components in an optimal order, i.e.,
one that generates all datasets that can be created through applying the transformers in \emph{any} order. 
To do this, we apply components of the same type in any order
and apply transformers of different types in the order $\miss$, $\flip$, $\fake$ (see Appendix).

\begin{example}
Suppose $\p=[\missB{2}{\pred_1}, \fakeB{1}{}]$. Then $\p(\T)$ is the set of all datasets obtained by adding up to $2$ arbitrary data points that satisfy $\pred_1$ to $\T$, and then removing any up to $1$ data point.
\end{example}



%% file: sections/certification.tex
\section{Certifying robustness for decision-tree learning}
\label{sec:abstract_certification_alg}
We begin with a simplified version of the CART algorithm \cite{cart}, which is our target for certification.

Given a dataset $\T$ and a Boolean function (predicate) $\phi: \calx \to \mathbb{B}$,
we define: 
$$\T_\phi = \{(x,y) \in \T \mid  \phi(x)\},$$ 
i.e., $\T_\phi$ is the set of elements satisfying $\phi$.
Analogously, $\T_{\neg\phi} = \{ (x,y)\in \T \mid  \neg \phi(x)\}$.

\begin{example}
Using $\phi\triangleq \textit{score}\leq 3$, we have $\toy_{\phi}=\{0,1,2,3\}$ and $\toy_{\neg\phi}=\{5,6,7,8,9\}$.
\end{example}

\paragraph{Learning algorithm}
To formalize our approach, it suffices to consider a simple algorithm that learns a decision stump, i.e., a tree of depth 1.
Therefore, the job of the algorithm is to choose a predicate (splitting rule) $\phi$
from a set of predicates $\Phi$ that optimally splits the dataset $\T$ into two datasets.
Formally, we define $\prob_i(T)$ as the proportion of $T$ with label $i$, i.e., 
\begin{equation}
    \prob_i(\T) =  |\{(x,i)\in \T\}|\ /\ {|T|}
\end{equation}
We use $\prob$ to calculate \emph{Gini impurity} ($\imp$), that is, $$\imp(\T) = \sum_{i=0}^{n-1} \prob_i(\T) (1-\prob_i(\T))$$ 

Using $\imp$, we assign each dataset-predicate pair a $\score$, where a low value indicates that $\phi$  splits $\T$ \emph{cleanly}, i.e., 
elements of $T_\phi$ (conversely, $T_{\neg\phi}$) have mostly the same label:
$$    \score(T,\phi) = |\T_\phi|\cdot \imp(\T_\phi) + |\T_{\neg \phi}|\cdot \imp(\T_{\neg\phi})
$$

Finally, we select the predicate that results in the lowest $\score$ (we break ties arbitrarily), as defined by the $\bestsplit$ operator:
$$
    \bestsplit(T) = \argmin_{\phi \in \Phi}\ \score(T,\phi)
$$

\begin{example}
For $\phi \triangleq \mathit{score}\leq 3$, $\score(\toy,\phi) = 4\times 0 + 5\times 0.32 = 1.6$.
\end{example}

\paragraph{Inference}
Given an optimal predicate $\phi$ and a new sample $x$ to classify,
we return the label with the highest proportion in the branch 
of the tree that $x$ takes.
Formally, 
$$
    \infer(T,\phi,x) = \argmax_i\ \prob_i(T_{\phi'}),
$$
where $\phi'$ is $\phi$
if $\phi(x) = \emph{true}$;
otherwise, $\phi'$ is $\neg\phi$.


\subsection{Certifying bias robustness with abstraction}\label{sec:abstract_prelims}
Given a dataset $\T$, bias model $\p$,
and sample $x$, our goal is to prove robustness (\cref{eq:rob}): no matter which dataset in $\p(\T)$
was used to learn a decision tree, the predicted label of $x$ is the same. 
Formally, 
\begin{equation}\label{eq:drob}
\text{there is a label $i$ s.t. for all }  \T' \in \p(\T),
\infer(\T',\phi',x) = i,\  \text{ where } \phi' = \bestsplit(\T')
\end{equation}
The na\"ive way to prove this is to learn a decision tree using each dataset in $\p(\T)$ and compare the results.
This approach is intractable or impossible, as $|\p(\T)|$ may be combinatorially large
or infinite.

Instead, we \emph{abstractly evaluate} the decision-tree-learning algorithm on the entire bias set $\p(\T)$ in a symbolic fashion, without having to enumerate all datasets.
Specifically, for each operator in the decision-tree-learning algorithm,
we define an abstract analogue, called an \emph{abstract transformer}~\cite{cousot1977abstract}, that operates over \emph{sets of training sets} symbolically. 
An abstract transformer is an approximation of
the original operator, in that it \emph{over-approximates}
the set of possible outputs on the set $\p(\T)$.

\paragraph{Sound abstract transformers}
Consider the $\prob_i$ operator, which takes a dataset
and returns a real number.
We define an abstract transformer $\aprob_i(\p(\T))$
that takes a set of datasets (defined as a bias set)
and returns an interval, i.e., a \emph{subset} of $\mathbb{R}$.
The resulting interval defines a range of possible values for the probability of class $i$.
E.g., an interval $\aprob_i(\p(\T))$ may be $[0.1,0.3]$,
 meaning that the proportion of $i$-labeled elements in datasets in $\p(\T)$ is
between $0.1$ and $0.3$, inclusive.

Given intervals computed by $\aprob_i$, downstream operators will be lifted into interval arithmetic, which is fairly standard. E.g., for $a,b,c,d\in\mathbb{R}$, $[a,b]+[c,d]=[a+c,b+d]$.
It will be clear from context when we are applying arithmetic operators to intervals.
For a sequence $\{x_i\}_i$ of intervals, $\argmax_i x_i$ returns a \emph{set of possible indices}, as intervals may overlap and there may be no unique maximum.

\begin{example}
Let $I = \{[1,2], [4,8], [6,7], [4,5]\}$. Then $\max(I) = \{[4,8], [6,7]\}$ because 6 is the \emph{greatest lower bound} of $I$, and $[4,8]$ and $[6,7]$ are the only intervals in $I$ that contain 6. 
\end{example}

For the entire certification procedure to be correct,
$\aprob$ and all other abstract transformers 
must be \emph{sound}. 
That is, they should over-approximate the set of 
possible outputs.
Formally,
$\aprob_i$ is a sound approximation of $\prob_i$ iff 
for all $i$ and all $\T' \in \p(\T)$, we have $\prob_i(\T') \in \aprob_i(\p(\T))$.

\paragraph{Certification process}
To perform certification,
we use an abstract transformer
$\abestsplit(\p(\T))$ to compute a \emph{set} of best predicates $\Phi^a$
for $\p(\T)$.
The reason $\abestsplit$ returns 
a set of predicates is because its input is a set of datasets that may result in different optimal splits.
Then, we use an abstract transformer
$\ainfer(\p(\T), \Phi^a, x)$ to compute a \emph{set} of 
labels for $x$.
If $\ainfer$ returns a singleton set,
then we have proven pointwise robustness for $x$ (\cref{eq:drob});
otherwise, we have an inconclusive result---we cannot falsify robustness because abstract transformers are over-approximate.

\subsection{Abstract transformers for \texorpdfstring{$\prob$}{pr}}\label{sec:abstraction}
We focus on the most challenging transformer, $\aprob$;
in \cref{sec:abstract_alg}, we show the rest of the transformers.

\paragraph{Abstracting missing data}
We begin by describing $\aprob$ for missing data bias,
 $\p = \missB{\m}{}$. From now on, we use $c_i$ to denote the number of samples $(x,y)\in T$ with $y=i$. We define $\aprob_i$ by considering how we can add data to \emph{minimize} the fraction of $i$'s in $T$ for the lower bound of the interval, and \emph{maximize} the fraction of $i$'s in $T$ for the upper bound. To minimize the fraction of $i$'s, we add $\m$ elements with label $j\neq i$; to maximize the fraction of $i$'s, we add $\m$ elements with label $i$. 
\begin{equation}\label{eq:cprob-missing}
    \aprob_i(\missB{\m}{}(\T)) = \left[ \frac{c_i}{|\T|+\m}\ ,\ \frac{c_i+m}{|\T|+\m}\right ]
\end{equation}

\begin{example}
Given $\p= \missB{1}{}$, we have $\aprob_{\color{ForestGreen}\checkmark}(\p(\toy)) = \left[\frac{4}{10},\frac{6}{10}\right]$.
\end{example}

\paragraph{Abstracting label-flipping}
Next, we define $\aprob_i$ for label-flipping bias,
where $\p = \flipB{\l}{}$.
Intuitively, we can minimize the proportion of $i$'s by flipping $\l$ labels from $i$ to $j\neq i$, and maximize the proportion of $i$'s by flipping $\l$ labels from $j\neq i$ to $i$. 
The caveat here is that if there are fewer than $\l$ of whichever label we want to flip, we are limited by $c_i$ or $\sum_{j\neq i} c_j$, depending on flipping direction. 
\begin{equation}\label{eq:cprob-labels}
    \aprob_i(\flipB{\l}{}(\T)) =  \left [ \frac{c_i-\min(\l,c_i)}{|\T|}, \frac{c_i+\min(\l,\sum_{j\neq i} c_j)}{|\T|}\right ]
\end{equation}

\begin{example}
Given $\p = \flipB{1}{}$, we have $\aprob_{\color{ForestGreen}\checkmark}(\p(\toy)) =  \left[\frac{4}{9},\frac{6}{9}\right]$.
\end{example}

Fake data bias models can be abstracted similarly (see Appendix).

\paragraph{Abstracting targeted bias models}
We now show how to abstract targeted bias models,
where a function $\pred$ restricts the affected samples. 
To begin, we limit $g$ to only condition on features, not the label.
In the case of $\miss^g$, the definition of $\aprob$ does not change, because even if we restrict the characteristics of the elements that we can add, we can still add up to $m$ elements with any label.

In the case of label-flipping, we constrain
the parameter $\l$ to be no larger than $|\T_\pred|$.
Formally, we define $\l_i=\min(\l, |\{(x,i)\in T: \pred(x)\}|)$ and then
\begin{equation}\label{eq:cprob-labels-targeted}
    \aprob_i(\flipB{\l}{\pred}(\T)) =  \left[ \frac{c_i-\l_i}{|\T|},\ \frac{c_i + \min(\sum_{j\neq i} \l_j, \l)}{|\T|} \right ]
\end{equation}

The definition for fake data is similar (see Appendix). 
The above definition is sound when $\pred$ conditions on the label;
however, the Appendix includes a more precise definition of $\aprob$ for 
that scenario.

\paragraph{Abstracting composite bias models}\label{sec:unified}
Now  consider a composite bias model consisting of all the basic bias models. Intuitively, $\aprob_i$ will need to reflect changes in $c_i$ that occur from adding data, flipping labels, and removing data. First, we consider a bias model with just one instance of each $\miss$, $\flip$, and $\fake$, i.e.,
$\p = [\missB{\m}{\pred_1}, \flipB{\l}{\pred_2}, \fakeB{\f}{\pred_3}]$. 
We define auxiliary variables $\l_i = \min(\l, |\{(x,i)\in T: \pred_2(x,i)\}|)$ and $\f_i = \min(\f, |\{(x,i)\in \T : \pred_3(x,i)\}|)$. Intuitively, these variables represent the number of elements with label $i$ that we can alter. Conversely, to represent the elements with a label other than $i$, we will use $\l'_i = \min(\l,\Sigma_{j\neq i} l_j)$ and $\f'_i=\min(\f,\Sigma_{j\neq i}c_j)$.
%
\begin{equation}\label{eq:cprob-unified}
    \aprob_i(\p(\T)) = \left[ \max\left(0, \frac{c_i-\l_i-\f_i}{|\T|-\f_i+\m}\right), \min\left(1, \frac{c_i+\l'_i+\m}{|\T|-\f'_i+\m}\right) \right]
\end{equation}

Extending the above definition to allow multiple uses of the same basic model,
e.g., $\{\flipB{\l_1}{\pred_1}, \flipB{\l_2}{\pred_2}\}$ is simple: essentially, we just sum $\l_1$ and $\l_2$. A full formal definition is in the Appendix. 

\begin{theorem}
$\aprob$ is a sound abstract transformer.
(In the Appendix, we also show that $\aprob$ is precise.)
\end{theorem}\label{thm:soundness}

\subsection{An abstract decision-tree algorithm}\label{sec:abstract_alg}
We define the remaining abstract transformers, with the goal of certification.
Our definitions are based on Drews et al.~\cite{drews-pldi};
the key difference is the $\T_\phi$
operation, which is dependent on the bias model.

\paragraph{Filtering}
We need $\afilter$, the abstract analogue of $T_\phi$. For $\flip$ and $\fake$, we define $\p(\T)_{\phi}$~= $\p(\T_{\phi})$. But for $\miss$, we have to alter the bias model, too, since after filtering on $\phi$ we only want to add new elements that satisfy $\phi$. We define $\missB{\m}{\pred}(\T)_{\phi} = \missB{\m}{\pred\land\phi} (\T_{\phi})$. Filtering composite bias models applies these definitions piece-wise (see a full definition and soundness proof in the Appendix). 

\paragraph{Gini impurity}
We lift $\imp$ to interval arithmetic: 
$\aimp(\T) = \sum_{i=1}^n \aprob_i(\T) ([1,1]-\aprob_i(\T))$. 

\paragraph{Cost}
Recall that $\score$ relies on $|\T_{\phi}|$. 
We want an abstract analogue of $|T_{\phi}|$
that represents the range of sizes of datasets in $\afilter$ and 
not the number of datasets in $\afilter$. 
To this end, we define an auxiliary function $\size$ where $\size(\afilter)=[a,b]$ such that
$a=\min \{|T'| : T'\in\afilter \}$ and
$b=\max\{|T'| : T'\in\afilter\}$. 

Then, we define the cost of splitting on $\phi$ as follows (recall that the operators use interval arithmetic):
\begin{equation}
    \ascore(\p(\T),\phi) = \size(\afilter)\times\aimp(\afilter) + \size(\afilterneg)\times\aimp(\afilterneg)
\end{equation}

\newcommand{\minimalInterval}{minimal interval}
\newcommand{\minInterval}{\mathsf{minInterval}}
\newcommand{\lub}{\mathit{lub}}
Since $\size$ and $\aimp$ return intervals, $\ascore$ will be an interval, as well. 

\paragraph{Best split}
To find the set of best predicates, we identify the \emph{least upper bound} ($\lub$) of any predicate's $\score$. Then, any predicate whose $\score$ overlaps with $\lub$ will be a member of the set of best predicates, too.
Formally, $\lub = \min_{\phi \in \Phi} b_\phi $, where $\ascore(\phi) = [a_\phi,b_\phi]$ 
Then, we define
$\abestsplit(\p(\T)) = \{\phi\in\Phi\mid a_\phi \leq \lub\}$.

\paragraph{Inference}
Finally, for inference, we evaluate every predicate computed by $\abestsplit$ on $x$ and collect all possible prediction labels. Intuitively, we break the problem into two pieces: first, we evaluate all predicates $\phi$ that satisfy $\phi(x)$ (i.e., when $x$ is sent down the left branch of the tree), and then predicates that satisfy $\neg\phi(x)$, (i.e., when $x$ is sent down the right branch of the tree). Formally, we compute:
\begin{equation}
    \ainfer(\p(\T),\Phi^a,x) = \underbrace{\bigcup_{\phi(x)} \argmax_i\
    \aprob_i(\p(\T)_\phi)}_{\textrm{labels for predicates $\phi$ s.t. $\phi(x)$}}
    \ \cup\  
    \underbrace{\bigcup_{\neg \phi(x)} \argmax_i\ \aprob_i(\p(\T)_{\neg\phi}}_{\textrm{labels for predicates $\phi$ s.t. $\neg\phi(x)$}})
 \end{equation}
where the range of $\cup$ is over predicates in $\Phi^a$.
Since our goal is to prove robustness,
we only care whether $| \ainfer(\p(\T),\Phi^a,x)| = 1$,
i.e., all datasets produce the same prediction.

\begin{theorem}\label{thm:soundness-entire}
If $| \ainfer(\p(\T),\Phi^a,x)| = 1$, where $\Phi^a = \bestsplit^a(\p(\T))$,
then $x$ is robust  (\cref{eq:drob}).
\end{theorem}

\begin{example}
Recall \cref{fig:running} 
with bias model $\p = \flip^g_1$, 
where $g$ targets Black people with $\color{Maroon}\times$ label.
$\abestsplit(\p(\toy))$ returns the singleton set $\Phi^a= \{\mathit{score} \leq 3\}$.
Then, given input $x=\langle\mathit{race}=\textrm{Black}, \mathit{score}=7\rangle$,
$\ainfer(\p(\toy), \Phi^a, x) = \{{\color{ForestGreen}\checkmark}\}$,
since $\aprob_{\color{ForestGreen}\checkmark}(\p(\toy)_{\mathit{score} > 3}) = [0.8,1]$, which is greater than 
$\aprob_{\color{Maroon}\times}(\p(\toy)_{\mathit{score} > 3}) = [0,0.2]$.
Therefore, the learner is robust on $x$.
\end{example}


%% file: sections/experiments.tex
\section{Experimental evaluation}
\label{experiments}
We implement our certification technique in C++
and call it \NAME, as it extends Antidote~\cite{drews-pldi} to programmable bias models. 
To learn trees with depth $> 1$, we apply the presented procedure
recursively.
We use Antidote's \emph{disjunctive domain}, which is beneficial
for certification~\cite{drews-pldi} but requires a large amount of memory because it keeps track of many different datasets on each decision-tree path.
%
We evaluate on Adult Income~\cite{uci-data} (training $n$=\num{32561}), COMPAS~\cite{compas} ($n$=\num{4629}), and Drug Consumption~\cite{drug-data} ($n$=\num{1262}). A fourth dataset, MNIST 1/7 ($n$=\num{13007}), is in the Appendix. For all datasets, we use the standard train/test split if one is provided; otherwise, we create our own train/test splits, which are available in our code repository at \url{https://github.com/annapmeyer/antidote-P}.





For each dataset, we choose the smallest tree depth where
accuracy improves no more than 1\% at the next-highest depth.
For Adult Income and MNIST 1/7, this threshold is depth 2 (accuracy 83\% and 97\%, respectively); for COMPAS and Drug Consumption it is depth 1 (accuracy 64\% and 76\%, respectively). 
We  run additional experiments on COMPAS and Drug Consumption at depths 2 and 3 to evaluate how tree depth influences \NAME's efficiency (see Appendix).

A natural baseline is enumerating all datasets in the bias set
but that is infeasible---see bias-set sizes in \cref{tab:all_data}.
To our knowledge, our technique (extended from~\cite{drews-pldi}), is the only method to certify bias robustness of decision-tree learners.

\subsection{Effectiveness at certifying robustness}
Table \ref{tab:all_data} shows the results. Each entry in the table indicates the percentage of test samples for which \NAME\ can prove robustness with a given bias model and the shading indicates the size of the bias set, $|\p(\T)|$. We see that even though the perturbation sets are very large---sometimes infinite---we are able to certify robustness for a significant percentage of elements.

\input{effectiveness-table}

\paragraph{By dataset} Certification rates vary from 98.8\% robustness for $\missB{0.05\%}{}$ on Adult Income
(i.e., the predictions of 98.8\% of the points in the test set do not change if up to 0.05\% new points are added to the training set)
to just 3\% robustness for $\flipB{1\%}{}$ on COMPAS. 
Even for a single bias model, the certification rates vary widely: under $\flipB{0.2\%}{}$, we can verify 94.5\% of samples as robust for Drug Consumption, but only 70.2\% for Adult Income and 47.8\% for COMPAS. 
We posit that these differences stem from inherent properties of the datasets.
The normalized $\score$ of the optimal top-most split is $0.30$ for Adult Income, $0.35$ for Drug Consumption, and $0.45$ for COMPAS (recall that a lower $\score$ corresponds to greater information gain).
As a result, biasing a fixed percentage of data yields greater instability for COMPAS, since the data already exhibited poorer separation.

\paragraph{By bias model} There are also differences in certification rates between bias models. 
$\flip$ is more destructive to robustness because flipping a single label results in a symmetric difference of 2 from the original dataset (as if we removed an element from the set and then inserted a new one with a flipped label),
while adding a single item results in a symmetric difference of 1.

The composite bias models display similar dataset- and bias model-based trends. Notably, $\miss$ + $\fake$ yields a lower certifiable-robustness rate than $\flip$. E.g., under $\flipB{0.1\%}{}$, we can certify 71.5\% of COMPAS test samples as robust. But for $\missB{0.1\%}{}$ + $\fakeB{0.1\%}{}$ (that is, 0.2\% bias total), we are only able to certify 50.5\% of test samples as robust. This shows that $\flip$ is a useful modeling tool for situations where we believe the features of all data points to be trustworthy, but suspect that some labels may be incorrect.
The targeted bias models allow for greater certification rates than the non-targeted versions; this is expected because they result in smaller bias sets.

In summary, \NAME\ \textbf{can effectively certify robustness
across a variety of bias models, but its success depends on properties of the dataset such as separability}.

\subsection{Demographic variations}\label{sec:demo}
We evaluated differences in certifiable-robustness rates across demographic groups in all three datasets. We present results from COMPAS and Adult Income in \cref{fig:demo} (results for Drug Consumption are in the Appendix; they are less interesting due to a lack of representation in the dataset).


\paragraph{COMPAS}
 \cref{fig:demo}a shows that under $\flip$, White women are robust at a higher rate than any other demographic group, and that Black men and women are the least robust. Notably, for $\flipB{0.4\%}{}$, we are able to certify robustness for 50.4\% of White women, but 0\% of Black people. There is also a significant gap between White women and White men at this threshold (50.4\% vs. 38.8\%). 
We can explain the gaps in certification rates of different subgroups by looking at the training data. In the COMPAS dataset, the same predicate provides the optimal split for every race-gender subgroup, but for White women the resulting split has $\score=0.41$ versus $\score=0.46$ for Black people. It is not clear whether this difference stems from sampling techniques or inherent differences in the population, but regardless, the end result is that \textbf{predictions made about Black people from decision trees trained on COMPAS are more likely to be vulnerable to data bias}. 

To validate that the disparities in certifiable-robustness rates are due to inherent dataset properties rather than the abstraction, we performed random testing by perturbing the COMPAS dataset to try to find robustness counterexamples, i.e., datasets in the bias set that yield conflicting predictions on a given input. We found more counterexamples to robustness for Black people than for White people, which is further evidence for our claim that the robustness disparities are inherent to the dataset. 

\input{demo_graphs}

\paragraph{Targeted bias models (COMPAS)}
If we choose $g \triangleq (\emph{race}=\textrm{Black} \land \emph{label}=\textrm{positive})$ (\cref{fig:demo}b) in $\flipB{}{g}$ to model the real-world situation where structural or individual racism can lead to increased policing and convictions among Black people in the U.S., then there are generally higher robustness rates at moderate bias levels (e.g., {$\sim$}50\% robustness for all demographic groups at 0.4\% poisoning).
%
However, as the amount of bias increases, 
a gap between White and Black certification rates emerges 
(in exact terms, 32.9\% of White test samples are certifiably robust versus 0\% of Black test samples starting at 1.1\% bias and continuing through, at least, 10.8\% bias). 
It is unclear whether this trend stems from inherent dataset properties, 
or is due to the over-approximate nature of the abstraction.

By contrast, using $g \triangleq (\emph{race}=\textrm{White} \land \emph{label}=\textrm{negative})$ (\cref{fig:demo}c) to describe that White people may be under-policed or under-convicted due to White privilege nearly eliminates discrepancies between demographic groups.
%
%
In particular, Black men (previously the least-robust subgroup) are the most robust of any population. 
$\flipB{}{g}$ and $\flipB{}{g'}$ differ only on how they describe societal inequities: are White people under-policed, or are Black people over-policed? 
However, the vast differences in demographic-level robustness rates between $\flipB{}{g}$ and $\flipB{}{g'}$ shows that \textbf{the choice of predicate  is crucial when using targeted bias models}. 
More experimentation is needed to understand why these results occur, and how consistent they are across different train/test splits of the data.
However, our preliminary results indicate that \NAME\ could be a useful tool for social scientists to understand how data bias can affect the reliability of machine-learning outcomes.

\paragraph{Adult Income} \cref{fig:demo}d shows robustness by demographic group for $\flip$. We see that Black men have about a 5\% lower robustness rate than other demographic groups and that at higher bias levels, White women also have about a 5\% lower robustness rate than White men or Black women. Using $\flipB{}{g}$ where $\pred = (\emph{race}=\textrm{Female} \land \emph{Label}=\textrm{negative})$ led to similar results (see Appendix). 

%% file: effectiveness-table.tex
\newcommand{\nottoobig}{NotTooBig} 
\newcommand{\prettybig}{PrettyBig} 
\newcommand{\verybig}{VeryBig} 
\newcommand{\gigantic}{Gigantic} 
\newcommand{\megagigantic}{MegaGigantic} 
\newcommand{\infinite}{Infinite}

\definecolor{Infinite}{rgb}{1,0.38,0.25}
\definecolor{MegaGigantic}{rgb}{1.0,0.50,0.36}
\definecolor{Gigantic}{rgb}{1.0,0.62,0.47}
\definecolor{VeryBig}{rgb}{1.0,0.74,0.58}
\definecolor{PrettyBig}{rgb}{1.0,0.86,0.69}
\definecolor{NotTooBig}{rgb}{1.0,0.98,0.80}

\begin{table}[t]
\small    \centering
    \caption{Certification rates for various bias models. Targeted bias models use predicates $(\emph{race}=\mathrm{Black} \text{ and } \emph{label}=\mathrm{positive})$ for COMPAS and $(\emph{gender}=\mathrm{female} \text{ and } \emph{label}=\mathrm{negative})$ for Adult Income. Composite models show cumulative bias, e.g., $0.2\%$ $\miss$ + $\fake$ bias equates to $0.1\%$ bias of each $\miss$ and $\fake$.
    Empty entries indicate tests that failed due to memory constraints (96GB).
    }
    \label{tab:all_data}
    \begin{tabular}{llrrrrrr}\toprule
    & & \multicolumn{6}{c}{\textbf{Bias amount as a percentage of training set}}\\\cmidrule(lr){3-8}
    \textbf{Bias type} & \textbf{Dataset} & 0.05 & 0.1 & 0.2 & 0.4 & 0.7 & 1.0 \\\midrule
   \multirow{5}{*}{\begin{tabular}[c]{@{}l@{}} \normalsize{$\miss$} \\ (missing data) \end{tabular}}
        & Drug Consumption      & 94.5\cellcolor{\prettybig} & 94.5\cellcolor{\prettybig}  & 94.5\cellcolor{\prettybig} & 94.5\cellcolor{\verybig}  & 85.1\cellcolor{\gigantic} & 85.1\cellcolor{\gigantic} \\
        & COMPAS                & 89.0\cellcolor{\prettybig} & 81.9\cellcolor{\prettybig} & 52.9\cellcolor{\prettybig} & 45.3\cellcolor{\verybig} & 9.3\cellcolor{\gigantic} & 9.2\cellcolor{\gigantic} \\
        & Adult Income (AI)         & 96.0\cellcolor{\infinite} & 86.9\cellcolor{\infinite} & 72.8\cellcolor{\infinite} & 60.9\cellcolor{\infinite} & & \\\cmidrule(lr){2-8}
        & COMPAS targeted       & 89.0\cellcolor{\prettybig} & 89.0\cellcolor{\prettybig} & 81.9\cellcolor{\prettybig} & 52.9\cellcolor{\verybig} & 47.8\cellcolor{\gigantic} & 42.3\cellcolor{\gigantic} \\
        & AI targeted & 98.8\cellcolor{\infinite} & 97.2\cellcolor{\infinite} & 86.6\cellcolor{\infinite} & 73.0\cellcolor{\infinite} & 62.0\cellcolor{\infinite} & 31.6\cellcolor{\infinite} \\\midrule  
        
    \multirow{5}{*}{\begin{tabular}[c]{@{}l@{}} \normalsize{$\flip$} \\ {(label-flipping)}\end{tabular}}  
        & Drug Consumption      & 94.5\cellcolor{\nottoobig} & 94.5\cellcolor{\nottoobig} & 94.5\cellcolor{\nottoobig} & 92.1\cellcolor{\prettybig} & 85.1\cellcolor{\prettybig} &  7.1\cellcolor{\prettybig} \\
        & COMPAS                & 81.9\cellcolor{\prettybig} & 71.5\cellcolor{\prettybig} & 47.8\cellcolor{\prettybig} & 20.6\cellcolor{\verybig} &  3.0\cellcolor{\verybig} &  3.0\cellcolor{\gigantic} \\
        & Adult Income          & 95.8\cellcolor{\verybig} & 72.9\cellcolor{\gigantic} & 70.2\cellcolor{\megagigantic} & 34.8\cellcolor{\megagigantic} & & \\\cmidrule(lr){2-8}
        & COMPAS targeted       & 89.0\cellcolor{\nottoobig} & 81.9\cellcolor{\prettybig} & 71.5\cellcolor{\prettybig} & 50.5\cellcolor{\prettybig} & 43.2\cellcolor{\verybig} & 24.2\cellcolor{\verybig} \\
        & AI targeted & 98.6\cellcolor{\verybig} & 97.0\cellcolor{\verybig} & 74.4\cellcolor{\gigantic} & 71.0\cellcolor{\gigantic} & 45.4\cellcolor{\gigantic} & 25.8\cellcolor{\megagigantic} \\\midrule
                  
    \multirow{3}{*}{\begin{tabular}[c]{@{}l@{}} \normalsize{$\miss$ + $\fake$}\\ {(missing + fake)}\end{tabular}}
        & Drug Consumption & 94.5\cellcolor{\prettybig} & 94.5\cellcolor{\prettybig} & 94.5\cellcolor{\prettybig} & 94.5\cellcolor{\verybig} & 85.1\cellcolor{\gigantic} & 85.1\cellcolor{\gigantic} \\
        & COMPAS           & 81.9\cellcolor{\prettybig} & 76.2\cellcolor{\prettybig} & 52.9\cellcolor{\prettybig} & 43.2\cellcolor{\verybig} &  9.3\cellcolor{\gigantic} &  9.3\cellcolor{\gigantic} \\
        & Adult Income     & 96.0\cellcolor{\infinite} & 95.6\cellcolor{\infinite} & 72.8\cellcolor{\infinite} & 68.3\cellcolor{\infinite} & 36.2\cellcolor{\infinite} & \\\midrule
        
   \multirow{3}{*}{\begin{tabular}[c]{@{}l@{}} \normalsize{$\miss$ + $\flip$} \\ {(missing + label-flipping)}\end{tabular}}
        & Drug Consumption & 94.5\cellcolor{\prettybig} & 94.5\cellcolor{\prettybig} & 92.1\cellcolor{\prettybig} & 92.1\cellcolor{\verybig} & 85.1\cellcolor{\gigantic} & 38.0\cellcolor{\gigantic} \\
        & COMPAS           & 81.9\cellcolor{\prettybig} & 71.5\cellcolor{\prettybig} & 50.5\cellcolor{\prettybig} & 41.6\cellcolor{\verybig} &  9.3\cellcolor{\gigantic} &  3.0\cellcolor{\gigantic} \\
        & Adult Income     & 95.9\cellcolor{\infinite} & 74.3\cellcolor{\infinite} & 71.1\cellcolor{\infinite} & 49.0\cellcolor{\infinite} & & \\\bottomrule
    \end{tabular}
    \begin{tabular}{lllllll}
         Bias-set size color scheme~~~~ & $<10^{10}$\cellcolor{\nottoobig} & $<10^{50}$\cellcolor{\prettybig} & $<10^{100}$\cellcolor{\verybig} & $<10^{500}$\cellcolor{\gigantic} & $>10^{500}$\cellcolor{\megagigantic} & infinite\cellcolor{\infinite} 
    \end{tabular}
\end{table}

%% file: demo_graphs.tex
\begin{figure}[t]
\centering
\tiny
\pgfplotsset{filter discard warning=false}

\pgfplotscreateplotcyclelist{whatever}{%
    black,thick,dashed,every mark/.append style={fill=blue!80!black},mark=none\\%
    gray,thick,dashed,every mark/.append style={fill=red!80!black},mark=none\\%
    black,thick,every mark/.append style={fill=blue!80!black},mark=none\\%
    gray,thick,every mark/.append style={fill=red!80!black},mark=none\\%
    }
    
\begin{tikzpicture}
    \begin{groupplot}[
            group style={
                group size=4 by 1,
                horizontal sep=.1in,
                vertical sep=.05in,
                ylabels at=edge left,
                yticklabels at=edge left,
                xlabels at=edge bottom,
                xticklabels at=edge bottom,
            },
            height=.8in,
            xlabel near ticks,
            ylabel near ticks,
            scale only axis,
            width=0.2*\textwidth,
            xtick={0,.2,.4,.6,.8,1,1.2,1.4},
            minor xtick={0,0.1,0.2,0.3,0.4,0.5,0.6,0.7,0.8,0.9,1,1.1,1.2,1.3},
            xmin=0,
            xmax=1.1
        ]

        \nextgroupplot[
            ylabel=Certifiable robustness (\%),
            ymin=0,
            ymax=100,
            cycle list name=whatever,
            xlabel=Bias amount (\%)]
        \addplot table [x=bias,y=whitemen, col sep=comma]{data/compas1_demo.csv};
        \addplot table [x=bias,y=whitewomen, col sep=comma]{data/compas1_demo.csv};
        \addplot table [x=bias,y=blackmen, col sep=comma]{data/compas1_demo.csv};
        \addplot table [x=bias,y=blackwomen, col sep=comma]{data/compas1_demo.csv};

        \nextgroupplot[
            ymin=0,
            ymax=100,
            xtick={0,.2,.4,.6,.8,1,1.2,1.4,1.6,1.8},
            minor xtick={0,0.1,0.2,0.3,0.4,0.5,0.6,0.7,0.8,0.9,1,1.1,1.2,1.3,1.4,1.5,1.6,1.7,1,8},
             cycle list name=whatever
        ]
        \addplot table [x=bias, y=whitemen, col sep=comma]{data/compas_labels_targeted_demo.csv};
        \addplot table [x=bias, y=whitewomen, col sep=comma]{data/compas_labels_targeted_demo.csv};
        \addplot table [x=bias, y=blackmen, col sep=comma]{data/compas_labels_targeted_demo.csv};
        \addplot table [x=bias, y=blackwomen, col sep=comma]{data/compas_labels_targeted_demo.csv};

        \nextgroupplot[
            ymin=0,
            ymax=100,
            ytick={0,20,40,60,80},
             cycle list name=whatever,
        ]
        \addplot table [x=bias, y=whitemen, col sep=comma]{data/compas_demo_targeted_backwards_demo.csv};
        \addplot table [x=bias, y=whitewomen, col sep=comma]{data/compas_demo_targeted_backwards_demo.csv};
        \addplot table [x=bias, y=blackmen, col sep=comma]{data/compas_demo_targeted_backwards_demo.csv};
        \addplot table [x=bias, y=blackwomen, col sep=comma]{data/compas_demo_targeted_backwards_demo.csv};

        \nextgroupplot[
            ymin=1,
            ymax=100,
            ytick={0,20,60,40,80},
            legend style={at={(1,0)},anchor=south},
             cycle list name=whatever,
        ]
        \addplot table [x=bias, y=whitemen, col sep=comma]{data/adult_labels_demo.csv};
        \addplot table [x=bias, y=whitewomen, col sep=comma]{data/adult_labels_demo.csv};
        \addplot table [x=bias, y=blackmen, col sep=comma]{data/adult_labels_demo.csv};
        \addplot table [x=bias, y=blackwomen, col sep=comma]{data/adult_labels_demo.csv};

\legend{White men,White women,Black men,Black women};

    
    \end{groupplot}
    \node[draw] at (2.45,1.7) {\normalsize a};
    \node[draw] at (5.55,1.7) {\normalsize b};
    \node[draw] at (8.55,1.7) {\normalsize c};
    \node[draw] at (11.55,1.7) {\normalsize d};

\end{tikzpicture}
\caption{Left to right: Certifiable robustness by demographic group on (a) 
COMPAS  under $\flipB{}{}$; (b) COMPAS under $\flipB{}{\pred}$ where $\pred \triangleq (\emph{race}=\textrm{Black} \land \emph{label}=\textrm{positive})$; (c)
COMPAS under $\flipB{}{\pred}$ where $\pred \triangleq (\emph{race}=\textrm{White} \land \emph{label}=\textrm{negative})$; (d) Adult Income under $\flipB{}{}$.}\label{fig:demo}
\end{figure}
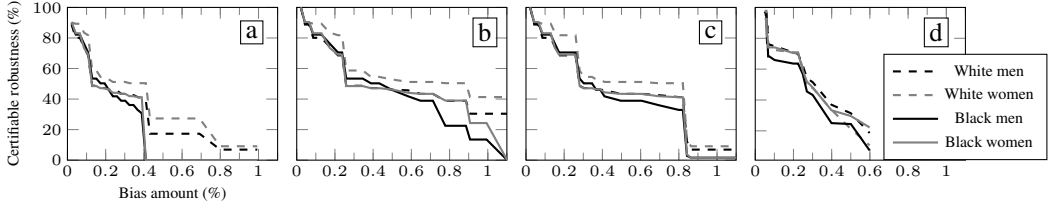

%% file: sections/conclusions.tex
\section{Conclusions and broader impacts}\label{sec:conclusions}
We saw that our decision-tree-learner abstraction is able to verify pointwise robustness over large and even infinite bias sets. 
These guarantees permit increased confidence in the trees' outputs because they certify that data bias has not affected the outcome (within a certain threshold). 
However, a couple of tricky aspects---and ones that we do not attempt to address---are knowing whether the assumptions underlying the bias model are correct, or whether our bias framework is even capable of representing all instances of real-world bias. 
If the user does not specify the bias model faithfully, then any proofs may not be representative.
Also, our tool only certifies robustness, not accuracy. 
Therefore, it may certify that a model will always output the \emph{wrong} label on a given data point. 
This behavior is linked to a shortcoming of many machine-learning audits: our tool cannot determine what is an appropriate use of machine learning.
Machine learning is often used to promote and legitimize uses of technology that are harmful or unethical. In particular, we want to call out our use of the COMPAS dataset: we feel that it is illustrative to show how certifiable-robustness rates can vary widely between different demographic groups and be sensitive to subtle shifts in the bias model. However, this use should not be taken as an endorsement for the deployment of recidivism-prediction models.
%

Another limitation is that our framework can only certify decision-tree learners. In practice, many machine learning applications use more sophisticated algorithms that we do not address here. Future work to generalize our ideas to other machine learning architectures would increase the utility of this style of robustness certification.

Returning to our work, \NAME\ has a place in data scientists' tool-kits as a powerful technique to understand robustness, and potential vulnerabilities, of data bias in decision-tree algorithms. 
An important direction for future work is to develop effective techniques for falsification of robustness (i.e., techniques to find minimal dataset perturbations that break robustness). We performed initial experiments in this area using brute-force techniques (i.e., randomly perturb data points, train a new decision tree and see whether the test sample's classification changes under the new tree---see the Appendix for more details). The results were promising in that we were able to find counter-examples to robustness for some data points, but there remain many data points that are neither certifiably robust via \NAME\ nor falsified as robust using random testing. Random testing was an interesting proof of concept, but we recommend that the future focus be on developing techniques to identify these dataset perturbations in a more scalable and intelligent way.
Other future work could also improve our approach's utility through tightening the analytical bounds, such as by abstracting over a more complex domain than intervals.

%% file: appendix/abstraction_details.tex
\section{Additional details and definitions}\label{app:abstraction}
Throughout the appendices, we use square brackets, rather than braces, to denote composite bias models: this is to emphasize that the transformers are ordered, and that alternate orderings often result in distinct bias sets.

\paragraph{Filtering composite bias models}
Filtering a composite bias model requires us to apply filter piece-wise, i.e., $[\missB{\m}{\pred_1}, \flipB{\l}{\pred_2}, \fakeB{\f}{\pred_3}](T)_{\phi} = [\missB{\m}{\pred_1\land \phi},\flipB{\l}{\pred_2}, \fakeB{\f}{\pred_3}](T_{\phi})$.

\paragraph{$\aprob$ for $\fake$}
Given $c_i$ samples in $\T$ with label $i$, we use $\f_i=\min(\f,c_i)$ and then define
\begin{equation}
    \aprob_i(\fakeB{\f}{}(\T)) = \left[ \frac{c_i-\f_i}{|T| - \f_i}, \frac{c_i}{|T|-\sum_{j\neq i} \f_j} \right]
\end{equation}

For the edge case where $c_i=|T|$ and $c_i\leq \f$ for any $i$, we define $\aprob_j(\T)=[0,1]$ for all $j\in[1,n]$. A similar edge case applies, when necessary, to the composite definition.

\paragraph{Optimizing $\aprob$ when $g$ looks at the label}\label{app:targeted_prob_precision}
If $\pred$ conditions on the label, then we can improve the precision of $\aprob$ by defining each component individually. Suppose $\pred(x,y)= y \in S\land\pred'(x)$, where $S\subset\{1,\cdots,n\}$ and $\pred'$ is a predicate that only conditions on features.

For $\missB{\m}{\pred}$, we define
\begin{equation}\label{eq:cprob_precise_miss}
    \aprob_i(\missB{\m}{\pred}) = \begin{cases}
          \left[\frac{c_i}{|T|},\frac{c_i+\m}{|T|+\m}\right]  &\text{if} \, i\in S \text{ and } |S|=1 \vspace{1mm}\\
          \left[\frac{c_i}{|T|+\m}, \frac{c_i+\m}{|T|+\m}\right]\quad &\text{if} \, i \in S \text{ and } |S|\geq 2 \vspace{1mm}\\
          \left[\frac{c_i}{|\T|+\m},\frac{c_i}{|T|}\right] \quad & \text{else} \\
     \end{cases}
\end{equation}

For $\flipB{\l}{\pred}$, we use $\l_{a_i}=\min(\l,|\{(x,y)\in \T\mid y=i\land \pred(x,y)\}|)$ and $\l_{b_i}=\min(\l, |\{(x,y)\in\T\mid y\neq i\land\pred(x,y)\}|)$. Then, we define
\begin{equation}\label{eq:cprob_precise_labels}
    \aprob_i(\flipB{\l}{\pred}) = 
    \begin{cases}
        \left[\frac{c_i-\l_{a_i}}{|T|},\frac{c_i}{|T|}\right] &\text{if} \, i\in S \text{ and } |S|=1 \vspace{1mm}\\
        \left[\frac{c_i-\l_{a_i}}{|T|},\frac{c_i+\l_{b_i}}{|T|}\right] &\text{if} \, i\in S \text{ and } |S| \geq 2 \vspace{1mm}\\
        \left[\frac{c_i}{|T|},\frac{c_i+\l_{b_i}}{|T|}\right] &\text{else} \\
    \end{cases}
\end{equation}

For $\fakeB{\f}{\pred}$, we use $\f_{a_i}=\min(\f,|\{(x,y)\in\T\mid y=i\land \pred(x,y)\}|)$  and $\f_{b_i} = \min(\f,|\{(x,y)\in\T\mid y\neq i\land\pred(x,y)\}|)$. Then, we define
\begin{equation}\label{eq:cprob_precise_fake}
    \aprob_i(\fakeB{\f}{\pred}) = 
    \begin{cases}
        \left[\frac{c_i-\f_{a_i}}{|T|-\f_{a_i}},\frac{c_i}{|T|}\right] &\text{if} \, i\in S \text{ and } |S|=1 \vspace{1mm}\\
        \left[\frac{c_i-\f_{a_i}}{|T|-\f},\frac{c_i}{|T|-\f_{b_i}}\right] &\text{if} \, i\in S \text{ and } |S|\geq 2 \vspace{1mm}\\
        \left[\frac{c_i}{|T|},\frac{c_i}{|T|-\f_{b_i}}\right] &\text{else} \\
    \end{cases}
\end{equation}

We prove that the above definitions are sound and precise in \cref{app:precision}.
If desired, the above definitions can be pieced together to provide a more precise definition for composite bias models. However, we limit ourselves to just  the singleton transformers because notation becomes very messy, as we have to keep track of many variables indicating how many data elements satisfy the various conditions.

\paragraph{$\aprob$ for composite bias models with multiple versions of the same transformer}
If a bias model contains multiple instances of the same transformer, e.g., $\p=[\flipB{l_1}{\pred_1}, \flipB{l_2}{\pred_2}]$, we can combine everything into a single transformer. Formally, given 
\begin{align}
    \p = [ \missB{\m_1}{\pred'_1},\ldots,\missB{\m_j}{\pred'_j}, \flipB{\l_1}{\pred'_{j+1}},\ldots,\flipB{\l_{p}}{\pred'_{j+p}},  \fakeB{\f_1}{\pred'_{j+p+1}},\ldots,\fakeB{\f_q}{\pred'_{j+p+q}} ]
\end{align}
we define $$\m = m_1 +\ldots + m_j$$
$$\pred_1=\pred'_1\lor \cdots\lor \pred'_j$$  $$\l_i=\min(\l_1+\ldots + \l_p, |\cup_{i\in[1,p]}  \T_{\pred_{j+i}(x,y)\land y=i}|)$$
$$\pred_2=\pred'_{j+1}\lor\cdots\lor\pred'_{j+p}$$ 
$$\f_i=\min(\f_1+\ldots + \f_p, |\cup_{i\in[1,q]} \T_{\pred_{j+p+i}(x,y)\land y=i}|)$$ 
and 
$$\pred_3 = \pred'_{j+p+1}\lor\cdots\lor\pred'_{j+p+1}$$
Then, we can use the formula shown in Equation \ref{eq:cprob-unified} to compute $\aprob$. We show in \cref{sec:soundness_precision} that these definitions are sound.

\paragraph{Size ($\size$)}
We define $\size(\missB{\m}{g}) = [|\T|,|\T|+\m]$, $\size(\flipB{\l}{\pred})=[|\T|,|\T|]$, and $\size(\fakeB{\f}{\pred})=[|\T|-\f,|T|]$. Putting this all together, we have $\size([\missB{\m}{\pred_1},\flipB{\l}{\pred_2},\fakeB{\f}{\pred_3}]) = [|T|-\f,|T|+\m]$.


%% file: appendix/composition.tex
\section{Proof of optimal composition of transformers}
As stated in \cref{sec:3}, when composing transformers we want to apply them in an order that results in the largest composite bias model. To illustrate the concept of composite bias models' relative size, consider $\p=[\missB{1}{\pred_1}, \flipB{1}{\pred_2}]$ and $\p'=[\flipB{1}{\pred_2}, \missB{1}{\pred_1}]$ where  $\pred_1\triangleq$(\emph{gender}=female $\land$ \emph{label}=1) and $\pred_2\triangleq$(\emph{gender}=female). I.e., $\p$ adds one data point subject to $\pred_1$ and then flips the label of one data point subject to $\pred_2$, whereas $\p'$ performs these two operations in the opposite order. Under $\p$, we can use $\missB{}{}$ to add the data point $(x,y)=\langle$\emph{gender}=female,\emph{label}=1$\rangle$ and then use $\flip$ to change $y$ to 0. However, under $\p'$, we cannot alter the point that $\miss$ adds, so $\p$ and $\p'$ are not equivalent. In this case, $\p$ can construct every dataset that $\p'$ can construct (but not vice-versa), so we write $\p'\subset \p$ and say that $\p$ is larger than $\p'$.

First we consider the case when there are multiple transformers of the same type.

\begin{theorem}\label{thm:order1}
The bias models $\p_1=[\missB{\m_1}{\pred_1}, \missB{\m_2}{\pred_2}]$ and $\p_2=[\missB{\m_2}{\pred_2}, \missB{\m_1}{\pred_1}]$ are equivalent (and likewise for $\fake$ and $\flip$, as long as no $\flip$ predicate conditions on the label).
\end{theorem}
\begin{proof}

\paragraph{Missing data} The choice of what missing data to add has no bearing on what is already in (or not in) the dataset. Thus we can add $m_1$ elements that satisfy $\pred_1$ followed by $\m_2$ elements that satisfy $\pred_2$, or do the operators in the reverse order, but the end result is the same.

\paragraph{Label-flipping} Suppose $\p = [\flipB{\l_1}{\pred_1}, \flipB{\l_2}{\pred_2}]$, where $\pred_1$ and $\pred_2$ do not condition on the label. We want to show that $\p$ is equivalent to $\p' = [\flipB{\l_2}{\pred_2}, \flipB{\l_1}{\pred_1}]$. 

Consider an arbitrary $T'\in\p(\T)$. Each data point in $\T'$ is either (1) untouched by $\flipB{\l_1}{\pred_1}$ and $\flipB{\l_2}{\pred_2}$, (2) flipped only by $\flipB{\l_1}{\pred_1}$, (3) flipped only by $\flipB{\l_2}{\pred_2}$, or (4) flipped by both $\flipB{\l_1}{\pred_1}$ and $\flipB{\l_2}{\pred_2}$. If (1), clearly this is obtainable by $\p'$ since we did nothing. If (2), then since the data point is untouched by $\flipB{\l_2}{\pred_2}$, the data point can be flipped uninterrupted by $\flipB{\l_1}{\pred_1}$ (similarly for (3)). If (4), then -- since neither $\pred_1$ nor $\pred_2$ conditions on the label nor specifies what the new label can be -- we can still flip the label twice and end up with the same configuration. The same arguments hold had we started with $T''\in\p'$.
 Therefore, $\p$ and $\p'$ are equivalent. 

\paragraph{Fake data} The argument for fake data is similar.
\end{proof}

We can extend the proof of \cref{thm:order1} to arbitrarily many transformers of the same type.

Note that if $\flip$ conditions on the label, this proof does not hold. To continue with the terminology from the proof, if $\pred_i\triangleq$(\emph{label}=$a$), then applying $\flipB{\l_j}{\pred_j}$ first to some element $(x,a)$ yields $(x,a')$, which may no longer eligible to be flipped by $\flipB{\l_i}{\pred_i}$. 

\newcommand{\opt}{$\mathsf{OPT}$} 

Next, we show that there is an optimal way to compose transformers of different types. We define \emph{optimal} as largest, that is, some $\p'$ is optimal compared to $\p$ if $\p\subseteq \p'$. In other words, this notation says that every dataset created by $\p$ can also be created by $\p'$. For the next theorem and its proof we assume there is only one instance of each transformer type; however, in conjunction with \cref{thm:order1} we can extend it to include multiple instances of the same transformer type.

\begin{theorem}
 $\p = [\miss,\flip,\fake]$ is the optimal order to apply the transformers $\miss$, $\flip$, and $\fake$ (i.e., any other ordering $\p'$ of these transformers will satisfy $\p'\subseteq \p$).
\end{theorem}
\begin{proof}
We will show that other orderings of $\miss$, $\flip$, $\fake$ do not produce any biased datasets that do not also occur in $[\miss,\flip,\fake]$. For conciseness, we will write $[\miss,\flip, \fake]$ as \opt. 

1. $[\miss,\fake,\flip]$: We consider the set of datasets $S$ achieved after applying $\miss$, $\fake$, and then $\flip$. Fix an arbitrary $T'\in S$. $T'$ was constructed from $T$ by some sequence of adding, removing, and flipping data points. We have these categories for (potential) data points in $T'$:
(1) untouched data points,
(2) added data points, 
(3) added then removed data points, 
(4) added then flipped data points,
(5) removed data points, and
(6) flipped data points.
(1), (2), (5), and (6) apply single (or no) operators, so clearly are also attainable through \opt. $\miss$ occurs before both $\flip$ and $\fake$ in \opt, so (3) and (4) are attainable, as well.

2. $[\flip,\miss,\fake]$: We consider the set of datasets $S$ achieved after applying $\flip$, $\miss$, and then $\fake$. Fix an arbitrary $T'\in S$. $T'$ was constructed from $T$ by some sequence of flipping, adding, and removing data points. We have these categories:
(1) untouched data points, 
(2) flipped data points, 
(3) flipped then removed data points, 
(4) added data points, 
(5) added then removed data points, and 
(6) removed data points. 
(1), (2), (4), and (6) apply single (or no) operators, so clearly they are also attainable through \opt. Since flipping and adding each come before removing in \opt, (3) and (5) are obtainable as well.

3. $[\flip,\fake,\miss]$: We consider the set of datasets $S$ achieved after applying $\flip$, $\fake$, and then $\miss$. Fix an arbitrary $T'\in S$. $T'$ was constructed from $T$ by some sequence of flipping, removing, and adding data points. We have these categories:
(1) untouched data points, 
(2) flipped data points, 
(3) flipped then removed data points, 
(4) removed data points, and
(5) added data points.
(1), (2), (4), and (5) apply single (or no) operators, so clearly they are also attainable through \opt. Since flipping comes before removing in \opt, (3) is obtainable as well.

4. $[\fake,\miss,\flip]$: We consider the set of datasets $S$ achieved after applying $\fake$, $\miss$, and then $\flip$. Fix an arbitrary $T'\in S$. $T'$ was constructed from $T$ by some sequence of removing, adding, and flipping data points. We have these categories:
(1) untouched data points, 
(2) removed data points, 
(3) added data points, 
(4) added then flipped data points, and
(5) added data points.
(1), (2), (4), and (5) apply single (or no) operators, so clearly they are also attainable through \opt. Since flipping comes before removing in \opt, (3) is obtainable as well.

5. $[\fake,\flip,\miss]$: We consider the set of datasets $S$ achieved after applying $\fake$, $\flip$, and then $\miss$. Fix an arbitrary $T'\in S$. $T'$ was constructed from $T$ by some sequence of removing, flipping, and adding data points. We have these categories:
(1) untouched data points, 
(2) removed data points, 
(3) flipped data points, and
(4) added data points, 
Each of these apply single (or no) operators, so clearly they are also attainable through \opt. 

We were not able to construct a dataset not also in \opt\ through any other ordering of the operators, therefore, \opt\ is optimal.
\end{proof}

%% file: appendix/soundness.tex
\section{Proofs of soundness}\label{sec:soundness_precision}

\paragraph{Proof of Theorem \ref{thm:soundness}} $\aprob$ is sound.

\begin{proof} We show $\miss$ as a simple example to illustrate our approach, and then we show the proof for composite bias. We omit the proofs for $\flip$ and $\fake$ because they (like $\miss$) are special cases of composite. 

\paragraph{Missing data}
Given a dataset $\T$ with $n$ classes, suppose our bias set is $\missB{\m}{\pred}(\T)$. Furthermore, suppose that $c_i$ samples in $\T$ have label $i$. We define $\m_i\in[0,\m]$ to be the number elements we add with label $i$, and $\m'_i=\sum_{j\neq i}m_j$. Then, we can write the proportion of $i$'s as a function 
\begin{equation}
    F(\m_i,\m'_i)=\frac{c_i+\m_i}{|\T|+\m_i+\m'_i}
\end{equation}
The partial derivatives of $F$ have values $\frac{\delta F}{\delta \m_i}>0$ and $\frac{\delta F}{\delta \m'_i}<0$ over the entire domain $[0,\m]$, therefore, any conclusions we draw over $\mathbb{R}$ will also apply over the discrete integer domain. Therefore, to minimize $F$ we choose $\m_i=0$ and $\m'_i=\m$, and do the reverse to maximize $F$. Thus, $F_{\min}=\frac{c_i}{|T|+m}$ and $F_{\max}=\frac{c_i+m}{|T|+m}$. Since $[F_{\min},F_{\max}]\subseteq \aprob_i(\missB{\m}{\pred}(\T))$, $\aprob$ is sound.





\paragraph{Composite} Given a dataset $\T$ with $n$ classes, suppose that our bias model is $\p = [\missB{\m}{\pred_1}, \flipB{\l}{\pred_2}, \fakeB{\f}{\pred_3}]$. Furthermore, suppose that in $\T$, $c_i$ samples have label $i$. 

First, we consider how many elements we can add, flip, or remove of each label. Under $\miss$, we can add $m$ of label $i$ for all $i$.
Under $\flip$, we can flip up to $l$ labels from label $i$ to some $j\neq i$, assuming that there are at least $l$ elements  $(x,y)$ that satisfy $\pred_2(x,y)$ and $y=i$. Similarly, the maximum number of labels we can flip from any label $j\neq i$ to $i$ is bounded both by $l$ and by the number of elements that satisfy $\pred_2$ and have label $j$. Note that the elements we can flip are not just limited to $T$: we can also flip any of the newly-added $m$ elements. Formally, we define
     $\l_{a_i} = \min(|\{(x,y)\in \T \mid y=i \land \pred_2(x,y)\}| + \m, \l)$, and
     $\l_{b_i} = \min(|\{(x,y)\in \T \mid y\neq i \land \pred_2(x,y)\}| + \m, \l)$.
Similarly, under $\fake$ we must also consider the elements added by $\miss$ and those flipped from $j\neq i$ to $i$ by $\flip$. Therefore, we define
$\f_{a_i} = \min(|\{(x,y)\in \T\mid y=i\land \pred_3(x,y)\}| + \m + \l_{b_i}, \f)$ and
$\f_{b_i} = \min(|\{(x,y)\in \T\mid y\neq i\land \pred_3(x,y)\}| + \m + \l_{a_i}, \f)$.

We will show that under the specified bias model, the proportion of $i$'s in the dataset is always between $(c_i-\l_{a_i}-\f_{a_i})/(|\T|-\f_{a_i}+\m)$ and $(c_i+\m+\l_{b_i})/(|T|-\f_{b_i}+\m)$. 

Intuitively, we consider how to modify the proportion of $i$'s in $\T$. This proportion decreases by (1) flipping elements from class $i$ to some class $j\neq i$, (2) removing elements of class $i$, and (3) adding elements of a class other than $i$. Therefore, the fraction is minimized by doing (1), (2), and (3) as much as the bias model allows.

Formally, let $\m_i\in[0,\m]$ be the number of elements added with label $i$, $\m'_i\in[0,\m] = \sum_{j\neq i}m_j$, $\l_i\in[0,\l_{a_i}]$ be the number of elements flipped from class $i$ to any other class,  $\l'_i\in[0,\l_{b_i}]=\sum_{j\neq i} l_i$,  $\f_i\in[0,\f_{a_i}]$ be the number of elements removed with label $i$, $\f'_i\in[0,\f_{b_i}]=\sum_{j\neq i} \f_i$.

Then, we can write the proportion of $i$'s as a function \begin{equation}
    F(\m_i,\m'_i,\l_i,\l'_i,\f_i,\f'_i) = \frac{c_i+\m_i-\l_i+\l'_i-\f_i}{|T|+\m_i+\m'_i-\f_i-\f'_i}
\end{equation} 
We consider the partial derivative of $F$ with respect to each variable. For all input in the domain, we have $\frac{\delta F}{\delta \m_i} > 0$, $\frac{\delta F}{\delta \m'_i}$ < 0, $\frac{\delta F}{\delta \l_i} < 0$, $\frac{\delta F}{\delta \l'_i} > 0$, $\frac{\delta F}{\delta \f_i} < 0$, and  $\frac{\delta F}{\delta \f'_i} > 0$.

Note that each partial derivatives is monotone over all values in the domain. Thus, they are also monotone over integers, so any conclusions we yield over the real numbers can be relaxed to integers, as well. To minimize $F$, we will maximize each variable whose partial derivative is negative, and minimize each variable whose partial derivative is positive. That is, we choose $\m_i=0$, $\m'_i=\m$, $\l_i=\l_{a_i}$, $\l'_i=0$, $\f_i=\f_{a_i}$, and $\f'_i=0$ to minimize $F$, yielding \begin{equation*}
  F_{\min} = \frac{c_i-\l_{a_i}-\f_{a_i}}{|T|+\m-\f_{a_i}}
\end{equation*} 
Conversely, to maximize $F$ we maximize each variable whose partial derivative is positive and minimize each variable whose partial derivative is negative, yielding \begin{equation*}
 F_{\max} = \frac{c_i+\m+\l_{b_i}}{|T|+\m-\f_{b_i}}
\end{equation*} 
$[F_{\min},F_{\max}]\subseteq\aprob_i(\p(\T)$, therefore, $\aprob$ is sound.

\paragraph{Multiple composite bias models} Suppose 
\begin{equation*}
    \p = [ \missB{\m_1}{\pred'_1},\ldots,\missB{\m_j}{\pred'_j}, \flipB{\l_1}{\pred'_{j+1}},\ldots,\flipB{\l_{p}}{\pred'_{j+p}},  \fakeB{\f_1}{\pred'_{j+p+1}},\ldots,\fakeB{\f_q}{\pred'_{j+p+q}} ]
\end{equation*} and
\begin{equation*}
    \p' = \left[\missB{\sum_{i\in[1,j]} \m_i}{\bigvee_{i\in[1,j]} \pred_i}, 
            \flipB{\sum_{i\in[1,p]} \l_i}{\bigvee_{i \in [j+1,j+p]} \pred_i},
            \fakeB{\sum_{i\in[1,q]} \f_i}{\bigvee_{i\in[j+p+1,j+p+q]} \pred_i} \right]
\end{equation*}
We want to show that if $(x,y)\in T\cup T'$ is altered by $\p$, it can be altered by $\p'$ (in other words, we want to show that $\p\subseteq \p'$).

Case 1: $(x,y)$ was added by a transformer $\missB{\m_i}{\pred_i}$ for $i\in[1,j]$. Therefore, $\pred_i(x)$ and $(\bigvee_{i\in[1,j]}\pred_i)(x)$, as well, so $x$ can be added by $\missB{\sum_{i\in[1,j]} \m_i}{\bigvee_{i\in[1,j]} \pred_i}$.

Case 2: $(x,y)\in \T$ was flipped to $(x,y')\in \T'$ by $\flipB{\l_i}{\pred_{j+i}}$ for $i\in[1,p]$. This means that $\pred_{j+i}(x)$, so by extension, $(\bigvee_{i\in[j+1,j+p]}\pred_i)(x)$, which means that $x$'s label can be flipped by $\flipB{\sum_{i\in[1,p]} \l_i}{\bigvee_{i \in [j+1,j+p]} \pred_i}$.

Case 3: $(x,y)\in \T$ was removed by $\flipB{\f_i}{\pred_{j+p+i}}$ for $i\in[1,q]$. This means that $\pred_{j+p+i}(x)$, and thus $\bigvee_{i\in[j+p+1,j+p+q]}\pred_i(x)$. Therefore $(x,y)$ can be removed by  $\fakeB{\sum_{i\in[1,q]} \f_i}{\bigvee_{i\in[j+p+1,j+p+q]} \pred_i}$.  

To conclude, any modification to $\T$ that we can make under $\p$ is also attainable under $\p'$, therefore, if $\p'$ is sound, then $\p$ is sound, as well.
\end{proof}

\begin{proposition}\label{prop:filter}
Abstract filtering is sound.
\end{proposition}
\begin{proof}
To prove soundness for filtering, we need to show that $\T'\in\p(\T)\implies \T'_{\phi}\in\p(\T)_{\phi}$.

\paragraph{Missing data} Consider $\T'\in\missB{\m}{\pred}(T)$. Since $\T'\in\missB{\m}{\pred}$, we have $T' = T\cup S$ where $|S|\leq m$ and $\forall (x,y)\in S.\pred(x,y)$. Therefore, $T'_{\phi} = T_{\phi}\cup S_{\phi}$. Since $S_{\phi}\subseteq S$, then $|S_{\phi}|\leq m$ and $\forall(x,y)\in S_{\phi}.\pred(x,y)$. then $T_{\phi}\cup S_{\phi}\in \missB{\m}{\pred\land\phi}$, satisfying the claim.

\paragraph{Label flipping} Consider $\T'\in\flipB{\l}{\pred}(\T)$. Since $\T'\in\flipB{\l}{\pred}(\T)$, we know that $\T'=R\cup S$ where $R\subseteq T$, $|S|\leq \l$, and $T = R\cup \{(x,y)\mid (x,y')\in S\}$. Additionally, we have $(x,y)\in S \implies \pred(x,y)$. Consider $T'_{\phi} = R_{\phi}\cup S_{\phi}$. Since $R\subseteq T$, then $R_{\phi}\subseteq T_{\phi}$. Since $(x,y)\in S\implies (x,y')\in T$  and $\phi$ does not condition on the label, then $(x,y)\in S_{\phi}\implies (x,y')\in T_{\phi}$. Since $S_{\phi}\subseteq S$, we have $|S_{\phi}|\leq \l$. In total, this means that $T_{\phi}\in\flipB{\l}{\pred}(T_\phi)$.

\paragraph{Fake data} See~\cite{drews-pldi}.

\paragraph{Composite}
Suppose $\p=[\missB{\m}{\pred_1},\flipB{\l}{\pred_2},\fakeB{\f}{\pred_3}]$. We want to show that $\p(\T)_{\phi}\subseteq \p'(\T_{\phi})$, where $\p'=[\missB{\m}{\pred_1\land \phi},\flipB{\l}{\pred_2},\fakeB{\f}{\pred_3}]$. Consider $T'\in \p(\T)$. Each $(x,y)\in T'$ satisfies either (i) $(x,y)\in\T$, (ii), $(x,y')\in\T$ (for $y'\neq y$), or (iii) $(x,y')\notin T$. If (i), then $x\in \T'_{\phi}\implies x\in \T_{\phi}$, and likewise for $\neg\phi$. If (ii), $x\in\T'_{\phi}\implies x\in \T_{\phi}$ (since $\phi$ ignores the label), and likewise for $\neg\phi$. If (iii), the $x$ was added by $\miss$. $x\in \T'_{\phi}\implies\phi(x)\implies$ $x$ can be added by $\missB{\m}{\pred_1\land \phi}$, and if $x\notin T_{\phi}$, this means that $x$ cannot be added by $\missB{\m}{\pred_1\land \phi}$. Finally, there is a fourth category of elements: those in $T\setminus T'$. If $x\in T\setminus\T'$, then $x$ was removed by $\fake$. If $\phi(x)$, then $x$ can be removed from $T$ to make $T'$, otherwise, $x$ cannot be removed from $T$, so it must also be contained in $T'$.

Thus we have shown that $T'_{\phi}$ can be constructed from $T_{\phi}$ using $\p'$, therefore, filtering is sound. 
\end{proof}

\begin{proposition}\label{prop:imp}
$\aimp$ is sound
\end{proposition}
\begin{proof}
To show that $\aimp$ is sound, we must show that $\T'\in\p(\T)\implies \imp(\T')\in\aimp(\p(\T))$. By \cref{thm:soundness}, $\prob(\T')\in\prob(\p(\T))$. It follows from interval arithmetic $\imp(\T')\in\aimp(\p(\T))$.
\end{proof}

\begin{proposition}\label{prop:size}
$\size$ is sound.
\end{proposition}

\begin{proof}
Given a bias model $\p = [\missB{\m}{\pred_1}$, $\flipB{\l}{\pred_2}, \fakeB{\f}{\pred_3}]$ and a dataset $\T$, we can write the size of $\T'\in\p(\T)$ as $|\T|-\f'+\m'$, where $\f'\in[0,\f]$ and $\m'\in[0,\m]$. Clearly, $|\T|-\f'+\m'$ is minimized by choosing $\f'=\f$ and $\m'=0$, and maximized by choosing $\f'=0$ and $\m'=\m$. Since $|\T|-\f\in\size(\p)$ and $|\T|+\m\in\size(\p)$, we see that $\size$ is sound.
\end{proof}

\begin{proposition}\label{prop:cost}
$\ascore$ is sound.
\end{proposition}
\begin{proof}
We need to show that if $\T'\in\p(\T)$ and $\phi\in\Phi$, then $\score(\T',\phi)\in\ascore(\p(\T),\phi)$. By \cref{prop:filter}, we know that $T'_{\phi} \in \p(\T)_{\phi}$, which means that by \cref{prop:size},  $|\T'_{\phi}|\in\size(\p(\T)_{\phi})$, and similarly we can derive that $|\T'_{\neg\phi}|\in\size(\p(\T)_{\neg\phi})$. Additionally, by \cref{prop:imp}, we know that $\imp(\T'_{\phi})\in\aimp(\p(\T)_{\phi})$ and $\imp(\T'_{\neg\phi})\in\aimp(\p(\T)_{\neg\phi})$. 
By the rules of interval arithmetic, if $a\in[a_0,a_1]$ and $b\in[b_0,b_1]$, then $ab\in[a_0,a_1]\times[b_0,b_1]$ and $a+b\in[a_0,a_1]+[b_0,b_1]$. Therefore we can conclude that $\score(\T',\phi)\in\ascore(\p(\T),\phi)$, i.e., $\ascore$ is sound.
\end{proof}

\begin{proposition}\label{prop:asplit}
$\abestsplit$ is sound.
\end{proposition}
\begin{proof}
We want to show that if $\T'\in\p(\T)$, then $\bestsplit(\T')\in\abestsplit(\p(\T))$. 

Suppose $\gamma=\bestsplit(\T')$. Then, $\forall \phi\in\Phi$, $\score(\T',\phi)\geq \score(\T',\gamma)$.

Define $\phi^*$ such that $\lub = ub(\score(T,\phi^*))$, where $ub$ takes the upper bound of an interval. Since $\phi^*\in\Phi$, $\score(\T',\gamma)\leq\score(\T',\phi^*)$. By \cref{prop:cost}, $\score$ is sound, therefore $\score(\T',\gamma)\in \ascore(\p(\T),\gamma)$ and $\score(T',\phi^*)\in\ascore(\p(\T),\phi^*)$. And thus we have $lb(\ascore(\p(\T),\gamma))\leq ub(\ascore(\p(\T),\phi^*)) = \lub$. Thus, $\gamma\in\abestsplit(\p(\T))$.
\end{proof}

\paragraph{Proof of Theorem \ref{thm:soundness-entire}}
\begin{proof}
Given $\Phi^a=\abestsplit(\p(\T))$, if $|\ainfer(\p(\T),\Phi^a,x)|=1$, then we know \begin{equation*}
    \exists y. \forall \phi\in\Phi^a.\begin{cases}
        \textrm{if } \phi(x) \textrm{ then } \argmax_{i} \aprob_i(\p(\T)_{\phi}) = y\\
        \textrm{if } \neg\phi(x) \textrm { then }\argmax_{i} \aprob_i(\p(\T)_{\neg\phi}) = y
    \end{cases}
\end{equation*}
Given $T'\in\p(\T)$, we know from \cref{prop:asplit} that $\bestsplit(T')\in\abestsplit(T) = \Phi^a(\T)$. Therefore, $\infer(\T', \bestsplit(\T'),x) =y$, so the algorithm is robust on $x$. (Note that we defined $\p(\T)$ such that $\T\in\p(\T)$, therefore, the original prediction is also $y$.)
\end{proof}

%% file: appendix/precision.tex
\section{Precision}\label{app:precision}
Intuitively, an abstraction is precise if the abstraction cannot be improved. 
Formally, our abstraction is precise iff, it is sound and given $\aprob_i(\p(\T))= [a_i,b_i]$, then for each $i$ there is some $T'\in \p(\T)$ such that $\prob_i(T') =a_i$ and some $T''\in\p(\T)$ such that $\prob_i(T'')=b_i$.

\begin{theorem}
$\aprob$ is precise for missing data, label-flipping, and fake data.
\end{theorem}
\begin{proof}
\paragraph{Missing data (non-targeted)} In the proof of Theorem \ref{thm:soundness}, we show that the minimum proportion of $i$'s is $\frac{c_i}{|\T| + \m}$ and the maximum proportion of $i$'s is $\frac{c_i+\m}{|T|+\m}$. Since these bounds are equal to $\aprob$'s minimum and maximum, the interval is precise.

Proofs for label-flipping and fake data (non-targeted) similarly follow from Theorem \ref{thm:soundness}.

\paragraph{Targeted} Next, we will show that the definitions of $\aprob$ provided in Equations \ref{eq:cprob_precise_miss}-\ref{eq:cprob_precise_fake} are precise. First, we must show that these definitions are sound (as soundness is a prerequisite for precision).

To show that \cref{eq:cprob_precise_miss} is sound, we need to consider three cases (we use $S$ such that $\pred(x,y)=y\in S\land \pred'(x)$): first, $i\in S$ and $|S|=1$. In this case, we can add up to $\m$ elements of class $i$, and no elements of class $j \neq i$. Therefore, the minimum proportion of $i$'s is the original proportion: $\frac{c_i}{|\T|}$, and the maximum is  $\frac{c_i+\m}{|\T|}$.  Second, $i\in S$ and $|S|\geq 2$. In this case, we can add elements with label $i$ but we can also add elements with label $j\neq i$. As such, the minimum proportion of $i$'s is achieved by adding $\m$ elements with label $j$, and the maximum proportion of $i$'s is achieved by adding $\m$ elements with label $i$. Therefore, the minimum proportion of $i$'s is $\frac{c_i}{|\T|+\m}$ and the maximum proportion of $i$'s is $\frac{c_i+\m}{|\T|+\m}$.

Proofs of soundness for Equations \ref{eq:cprob_precise_labels} and \ref{eq:cprob_precise_fake} follow similarly.

To show precision, note that in the soundness proof we described exactly how to achieve the minimum and maximum bounds of $\aprob_i$. As such, we have shown that the $\aprob$ definition in \cref{eq:cprob_precise_miss} is precise. (Similar conclusions can be drawn based on the proofs for label-flipping and fake data.)
\end{proof}

$\aprob$ is not precise for composite bias because the auxiliary variables $\l_{a_i}, \l_{b_i},\f_{a_i}$, and $\f_{b_i}$ are over-approximate. As a motivating example, consider a bias model $[\missB{1}{},\flipB{2}{\pred}]$ over dataset $T$ with classes $\{0,1\}$ where $c_0=2$, $c_1=5$, and $|\{(x,y)\in \T : \pred(x,y) \land i=0\}| = 1$. By definition, $l_{a_0}=2$, yielding a minimum proportion of 0's to be $\frac{2-2}{7+2}=0$. However, the precise lower bound is $\frac{1}{9}$: to minimize the proportion of 0's, we add 2 elements with label 1 and flip the element that satisfies $\pred$ from label 0 to label 1.

%% file: appendix/experiments.tex
\section{Additional experimental data}
\subsection{MNIST 1/7 Binary}
We used MNIST 1/7 (the limitation of MNIST to just 1's and 7's, with training $n$=\num{13007}, as has been used in works including \cite{drews-pldi, steinhardt-certified}). We round each pixel to 0 or 1 (i.e., used a black-and-white image rather than a grayscale one). The accuracy of MNIST 1/7 binary (97.4\% at depth 2) is comparable to that of MNIST 1/7, but the time and memory requirements on \NAME\ are much less. 

\cref{tab:mnist_data} shows effectiveness data for MNIST 1/7 binary. We see that we are able to achieve high robustness certification rates, despite incredibly large perturbation set sizes. Notable, for $\missB{0.1\%}{}$, we achieve 100\% robustness even with a perturbation set size of over $10^{3058}$, and for $\missB{1\%}{}$, we achieve 68\% robustness with a perturbation set size larger than $10^{30460}$.

\input{effectiveness-mnist}

\subsection{Performance}\label{sec:performance}
We performed experiments on an HTCondor system, allowing us to perform many experiments in parallel. Each experiment ran robustness tests on a given bias model and dataset for between one and \num{1000} test samples (depending on the bias model and dataset). We used a single CPU for each experiment, and requested between 1 and 96GB of memory, depending on the bias model and dataset.

\paragraph{Bias model} 
Time and memory requirements increase exponentially as the amount of bias increases, as shown in \cref{tab:perf} for the Adult Income dataset under the $\flip$ bias model.  
Other datasets typically required less than 100 ms per test sample.
Additionally, bias models that yield lower certifiable robustness for a given bias threshold have correspondingly larger time and memory requirements
(e.g., 810s and 9.7GB of memory to yield 34.8\% robustness for $\flip_{0.4\%}$ as compared with 
77s and 1.3GB of memory to yield 60.3\% robustness for $\miss_{0.4\%}$).

\begin{table}[t]
\small
\centering
\caption{Time and memory requirements for certifying a single test sample under different $\flip$ bias models on Adult Income with depth=2.}
\label{tab:perf}
\begin{tabular}{l|rrrrrrr} \toprule
\multicolumn{1}{c}
{\multirow{2}{*}
{\begin{tabular}[c]{@{}c@{}} \end{tabular}}} 
& 
\multicolumn{6}{c}{Poisoning Amount (\%)} & \\
\multicolumn{1}{c}{} & 
    0.1 &    0.2 &    0.3 & 0.4 &    0.5 &    0.6  \\ \midrule 
Time (s.)  &  0.60 & 73.9 & 210 &    810 & 1800 & 5200  \\
Memory (GB)  &    0.01 & 0.8 & 3.6 &  9.7 & 21 &  60  \\
\bottomrule   
\end{tabular}
\end{table}

\paragraph{Datasets}
The size and complexity of the feature space is most influential in determining time and memory requirements. 
Experiments on the Adult Income dataset were more resource-intensive than those on Drug Consumption or COMPAS, a fact that can be explained by Adult Income having more unique feature values than the other datasets (22,100 for Adult Income vs. 219 for Drug Consumption and 53 for COMPAS). 
For each unique value of any feature, the algorithm checks an additional predicate, which explains the additional time and memory needs.

\paragraph{Complexity of decision-tree algorithm} 
Increasing the depth of the decision tree not only requires additional time to essentially re-run the algorithm at each internal node, 
but also leads to lower certifiable-robustness rates, as shown in \cref{tab:depths}.
This is because we must assume worst-case bias in each node. 
Intuitively, a depth 2 tree with 0.1\% bias may initially split the data into two children nodes, each with 50\% of the data. 
Our abstraction captures both the case where all bias occurs in the left child, and the case when all bias occurs in the right child. 
Therefore, we end up with an effective bias rate of 0.2\% in either child, yielding lower robustness.

\begin{table}[t]
\small
\centering
\caption{Robustness certification rates of COMPAS and Drug-Consumption datasets under $\flip$ for different decision-tree depths.}
\label{tab:depths}
\begin{tabular}{l|rrr|rrr} \toprule
\multicolumn{1}{c}{\multirow{2}{*}{\begin{tabular}[c]{@{}c@{}}Poisoning \\ amount (\%)\end{tabular}}} & \multicolumn{3}{c}{COMPAS} & \multicolumn{3}{c}{Drug Consumption} \\
\multicolumn{1}{c}{} 
      & depth 1 &    2 &    3 & depth 1 &    2 &    3  \\ \midrule 
0.10  &    71.5 & 51.5 & 34.0 &    94.5 & 83.8 &  8.5  \\
0.20  &    47.8 & 27.7 & 23.9 &    94.5 & 55.9 &  4.5  \\
0.50  &     9.3 &  2.5 &  0.7 &    85.1 & 27.5 &  0.5  \\
1.00  &     3.0 &  0.7 &    0 &     7.1 &  0.8 &    0  \\ \bottomrule   
\end{tabular}
\end{table}

\subsection{Additional experimental data}\label{app:demo_results}
\paragraph{General data}
\cref{fig:all_data} shows the certifiable robustness rates for each dataset and each main bias model ($\miss$, $\flip$, $\miss$ + $\fake$, and $\miss$ + $\flip$). 

\input{extra_graphs}

\paragraph{Demographic data}

\input{extra_demo_graphs}

\cref{fig:extra_demo} shows robustness levels stratified by demographic groups on various bias models. 
We see that COMPAS under $\missB{}{}$ (\cref{fig:extra_demo}b) displays similar robustness gaps to $\flipB{}{}$; namely, White people, and particularly White women, are robust at a higher rate than Black people. 
Adult Income under $\missB{}{}$ (\cref{fig:extra_demo}c) and under $\flipB{}{\pred}$ where $\pred\triangleq(\emph{gender}=\textrm{female}\land\emph{label}=\textrm{negative}$) (\cref{fig:extra_demo}d) behaves similarly to Adult Income under $\flipB{}{}$ (\cref{fig:demo}). 
That is, all demographic groups have roughly comparable robustness rates. 
Drug consumption under $\flipB{}{}$ (\cref{fig:extra_demo}a) yields comparable robustness rates between men and women (we do not graph robustness rates by racial group because the dataset is over 91\% White). 

\subsection{Additional details on random testing}\label{app:fuzzing}
On a random subset of 100 test elements from the COMPAS dataset, we tested \num{10000} dataset perturbations under $\flipB{0.5\%}{}$, $\flipB{1\%}{}$, $\flipB{2\%}{}$, and $\flipB{3\%}{}$. The number of elements for which we found a counter-example to robustness (i.e., a dataset perturbation that resulted in a different classification) is shown in \cref{tab:fuzzing}. We see that we are able to find counterexamples to robustness for a non-trivial portion of test samples. However, the gap between certified-robustness and proved-non-robust rates is still wide (the gap ranges from 86.7\% for $\flipB{0.5\%}{}$ to 68.0\% for $\flipB{3\%}{}$). As a result, there are many test samples that we cannot prove robustness for, but cannot find counterexamples for either. Future work to use a more precise abstract domain, or to better identify counterexamples to robustness could help to narrow this gap.

\begin{table}[t]
    \small
    \centering
    \caption{Number of elements with counterexamples to robustness after \num{10000} iterations of random testing from a subset of 100 test samples from COMPAS. All bias models are $\flip$, and bias level refers to number of affected elements as a percentage of training dataset size.}
    \label{tab:fuzzing}
    \begin{tabular}{cr} \toprule
        Bias level & \# of elements with counterexample  \\\midrule
        0.5 & 4 \\
        1.0 & 15 \\
        2.0 & 27 \\
        3.0 & 32 \\\bottomrule
    \end{tabular}
\end{table}

Breaking down the results for $\flipB{3\%}{}$ further, we found counterexamples to robustness for 50\% of Black women, 37\% of Black men, 29\% of White women, and 27\% of White men. Similarly, we found more counter-examples to robustness for Black people using other bias models. The empirical result of having more counterexamples for test instances representing Black people combined with the fact that we are able to certify a smaller percentage of test instances representing Black people (\cref{sec:demo}) suggests that the robustness differences are inherent to the data, rather than a property of the abstraction.

%% file: effectiveness-mnist.tex

\definecolor{Infinite}{rgb}{1,0.38,0.25}
\definecolor{MegaGigantic}{rgb}{1.0,0.50,0.36}
\definecolor{Gigantic}{rgb}{1.0,0.62,0.47}
\definecolor{VeryBig}{rgb}{1.0,0.74,0.58}
\definecolor{PrettyBig}{rgb}{1.0,0.86,0.69}
\definecolor{NotTooBig}{rgb}{1.0,0.98,0.80}

\begin{table}[t]
\small
    \centering
    \caption{Certification rates of MNIST 1/7 binary for various bias models, using decision trees of depth 2. The composite bias models show cumulative bias, e.g., 0.2\% $\miss$ + $\fake$ bias equates to 0.1\% bias of each $\miss$ and $\fake$. Note that the scale for perturbation set size is slightly different (larger) than that in Table \ref{tab:all_data}.}
    \label{tab:mnist_data}
    \begin{tabular}{llrrrrrr}\toprule
    & & \multicolumn{6}{c}{Bias amount as a percentage of training set size}\\\cmidrule(lr){3-8}
    Bias type & Dataset & 0.05 & 0.1 & 0.2 & 0.4 & 0.7 & 1.0 \\\midrule
   \multirow{1}{*}{\begin{tabular}[c]{@{}l@{}} \large{$\miss$} \end{tabular}}
        &  MNIST-1-7 binary      & 100.0\cellcolor{\gigantic} & 100.0\cellcolor{\gigantic} & 93.0\cellcolor{\gigantic} & 88.0\cellcolor{\megagigantic} & 85.0\cellcolor{\megagigantic} & 68.0\cellcolor{\megagigantic} \\\midrule  
        
    \multirow{1}{*}{\begin{tabular}[c]{@{}l@{}} \large{$\flip$} \end{tabular}}  
        &  MNIST-1-7 binary      & 100.0\cellcolor{\prettybig} & 95.0\cellcolor{\prettybig} & 88.0\cellcolor{\prettybig} & 70.0\cellcolor{\verybig} & 63.0\cellcolor{\verybig} & 38.0\cellcolor{\verybig} \\\midrule 
                  
    \multirow{1}{*}{\begin{tabular}[c]{@{}l@{}} \large{$\miss$ + $\fake$}\end{tabular}}
        &  MNIST-1-7 binary      & 100.0\cellcolor{\verybig} & 100.0\cellcolor{\verybig} & 93.0\cellcolor{\gigantic} & 88.0\cellcolor{\gigantic}  & 85.0\cellcolor{\megagigantic} & 68.0\cellcolor{\megagigantic} \\\midrule  
        
   \multirow{1}{*}{\begin{tabular}[c]{@{}l@{}} \large{$\miss$ + $\flip$}\end{tabular}}
        &  MNIST-1-7 binary      & 100.0\cellcolor{\verybig} & 95.0\cellcolor{\verybig} & 91.0\cellcolor{\gigantic} & 87.0\cellcolor{\gigantic} & 67.0\cellcolor{\megagigantic} & 62.0\cellcolor{\megagigantic} \\\bottomrule
    \end{tabular}
    \begin{tabular}{lllllll}
         Perturbation set size: & $<10^{10}$\cellcolor{\nottoobig} & $<10^{100}$\cellcolor{\prettybig} & $<10^{1000}$\cellcolor{\verybig} & $<10^{10000}$\cellcolor{\gigantic} & $>10^{10000}$\cellcolor{\megagigantic} & infinite\cellcolor{\infinite} \\\bottomrule
    \end{tabular}
\end{table}

%% file: extra_graphs.tex
\begin{figure}[htbp]
\centering
\tiny
\pgfplotsset{filter discard warning=false}

\pgfplotscreateplotcyclelist{whatever}{%
    black,thick,every mark/.append style={fill=blue!80!black},mark=none\\%
    }
    
\begin{tikzpicture}
    \begin{groupplot}[
            group style={
                group size=4 by 4,
                horizontal sep=.1in,
                vertical sep=.12in,
                ylabels at=edge left,
                yticklabels at=edge left,
                xlabels at=edge bottom,
                xticklabels at=edge bottom,
            },
            height=.8in,
            xlabel near ticks,
            ylabel near ticks,
            scale only axis,
            width=0.2*\textwidth,
            xtick={0,.2,.4,.6,.8,1,1.2,1.4},
            minor xtick={0,0.1,0.2,0.3,0.4,0.5,0.6,0.7,0.8,0.9,1,1.1,1.2,1.3},
            ytick={0,20,40,60,80,100},
            xmin=0,
            xmax=1.1
        ]
        
        \nextgroupplot[
            ylabel=$\missB{}{}$,
            ymin=0,
            ymax=100,
            xmax=2.2,
            title=Drug Consumption, xtick={0,.4,0.8,1.2,1.6,2},            
            minor xtick={0,0.2,0.4,0.6,0.8,1,1.2,1.4,1.6,1.8,2.0,2.2,2.4},
             ytick={0,20,40,60,80,100},
            cycle list name=whatever]
        \addplot table [x=bias,y=missing, col sep=comma]{data/drugs_d1_all.csv};
        
        \nextgroupplot[
            ymin=0,
            ymax=100,
            title=COMPAS,
            cycle list name=whatever,
         ]

        \addplot table [x=bias_pct,y=missing, col sep=comma]{data/compas_d1_all.csv};
        
        \nextgroupplot[
            ymin=0,
            ymax=100,
            ytick={0,20,40,60,80,100},
            xmax=0.8,
            title=Adult Income,
             cycle list name=whatever,
        ]
        \addplot table [x=bias,y=missing, col sep=comma]{data/adult_d2_all.csv};
        
         \nextgroupplot[
            ymin=0,
            ymax=100,
            xmax=2.2,
            xtick={0,.4,0.8,1.2,1.6,2},
            minor xtick={0,0.2,0.4,0.6,0.8,1,1.2,1.4,1.6,1.8,2.0},
            title=MNIST 1/7 Binary,
             cycle list name=whatever,
        ]
        \addplot table [x=bias,y=missing, col sep=comma]{data/mnist_all.csv};

        \nextgroupplot[
            ymin=0,
            ylabel=$\flipB{}{}$,
            ymax=100,
            xmax=2.2,
            xtick={0,.4,0.8,1.2,1.6,2},            minor xtick={0,0.2,0.4,0.6,0.8,1,1.2,1.4,1.6,1.8,2.0,2.2,2.4},
            ytick={0,20,60,40,80,100},
             cycle list name=whatever,
        ]
        \addplot table [x=bias,y=labels, col sep=comma]{data/drugs_d1_all.csv};
        
        \nextgroupplot[
            ymin=0,
            ymax=100,
            ytick={0,20,60,40,80,100},
             cycle list name=whatever,
        ]
        \addplot table [x=bias_pct,y=labels, col sep=comma]{data/compas_d1_all.csv};
        
        \nextgroupplot[
            ymin=0,
            ymax=100,
            ytick={0,20,40,60,80,100},
            xmax=0.8,
             cycle list name=whatever,
        ]
        \addplot table [x=bias,y=labels, col sep=comma]{data/adult_d2_all.csv};
        
        \nextgroupplot[
            ymin=0,
            ymax=100,
            xmax=2.2,
            xtick={0,.4,0.8,1.2,1.6,2},
            minor xtick={0,0.2,0.4,0.6,0.8,1,1.2,1.4,1.6,1.8,2.0},
             cycle list name=whatever,
        ]
        \addplot table [x=bias,y=labels, col sep=comma]{data/mnist_all.csv};
        
        \nextgroupplot[
            ymin=0,
            ymax=100,
            ylabel=$\missB{}{} + \fakeB{}{}$,
           xmax=2.2, xtick={0,.4,0.8,1.2,1.6,2},
            minor xtick={0,0.2,0.4,0.6,0.8,1,1.2,1.4,1.6,1.8,2.0,2.2,2.4},
             cycle list name=whatever
        ]
        \addplot table [x=bias,y=missing_fake, col sep=comma]{data/drugs_d1_all.csv};
        
        \nextgroupplot[
            ymin=0,
            ymax=100,
            xtick={0,.2,.4,.6,.8,1,1.2,1.4,1.6,1.8},
            minor xtick={0,0.1,0.2,0.3,0.4,0.5,0.6,0.7,0.8,0.9,1,1.1,1.2,1.3,1.4,1.5,1.6,1.7,1,8},
             cycle list name=whatever
        ]
        \addplot table [x=bias_pct,y=missing_fake, col sep=comma]{data/compas_d1_all.csv};
        
        \nextgroupplot[
            ymin=0,
            ymax=100,
            ytick={0,20,40,60,80,100},
            xmax=0.8,
             cycle list name=whatever,
        ]
        \addplot table [x=bias,y=missing_fake, col sep=comma]{data/adult_d2_all.csv};
        
        \nextgroupplot[
            ymin=0,
            ymax=100,
            xmax=2.2,
            xtick={0,.4,0.8,1.2,1.6,2},
            minor xtick={0,0.2,0.4,0.6,0.8,1,1.2,1.4,1.6,1.8,2.0},
             cycle list name=whatever,
        ]
        \addplot table [x=bias,y=missing_fake, col sep=comma]{data/mnist_all.csv};

        \nextgroupplot[
            ymin=0,
            ymax=100,
            xmax=2.2,
           ylabel=$\missB{}{} + \flipB{}{}$, xtick={0,.4,0.8,1.2,1.6,2},
            minor xtick={0,0.2,0.4,0.6,0.8,1,1.2,1.4,1.6,1.8,2.0,2.2,2.4},
            ytick={0,20,40,60,80,100},
             cycle list name=whatever,
             xlabel=Bias amount (\%)
        ]
        \addplot table [x=bias,y=missing_labels, col sep=comma]{data/drugs_d1_all.csv};

        \nextgroupplot[
            ymin=0,
            ymax=100,
            ytick={0,20,40,60,80,100},
             cycle list name=whatever,
             xlabel=Bias amount (\%)
        ]
        \addplot table [x=bias_pct,y=missing_labels, col sep=comma]{data/compas_d1_all.csv};
        
        \nextgroupplot[
            ymin=0,
            ymax=100,
            ytick={0,20,40,60,80,100},
             cycle list name=whatever,
             xmax=0.8,
            xlabel=Bias amount (\%)
        ]
        \addplot table [x=bias,y=missing_labels, col sep=comma]{data/adult_d2_all.csv};
        
        \nextgroupplot[
            ymin=0,
            ymax=100,
            xmax=2.2,
            xtick={0,.4,0.8,1.2,1.6,2},
            minor xtick={0,0.2,0.4,0.6,0.8,1,1.2,1.4,1.6,1.8,2.0},
             cycle list name=whatever,
            xlabel=Bias amount (\%)
        ]
        \addplot table [x=bias,y=missing_labels, col sep=comma]{data/mnist_all.csv};
        
    \end{groupplot}

\end{tikzpicture}
\caption{Certifiable robustness (shown on y-axis as a percentage of test data) for various datasets under different bias models.}\label{fig:all_data}
\end{figure}
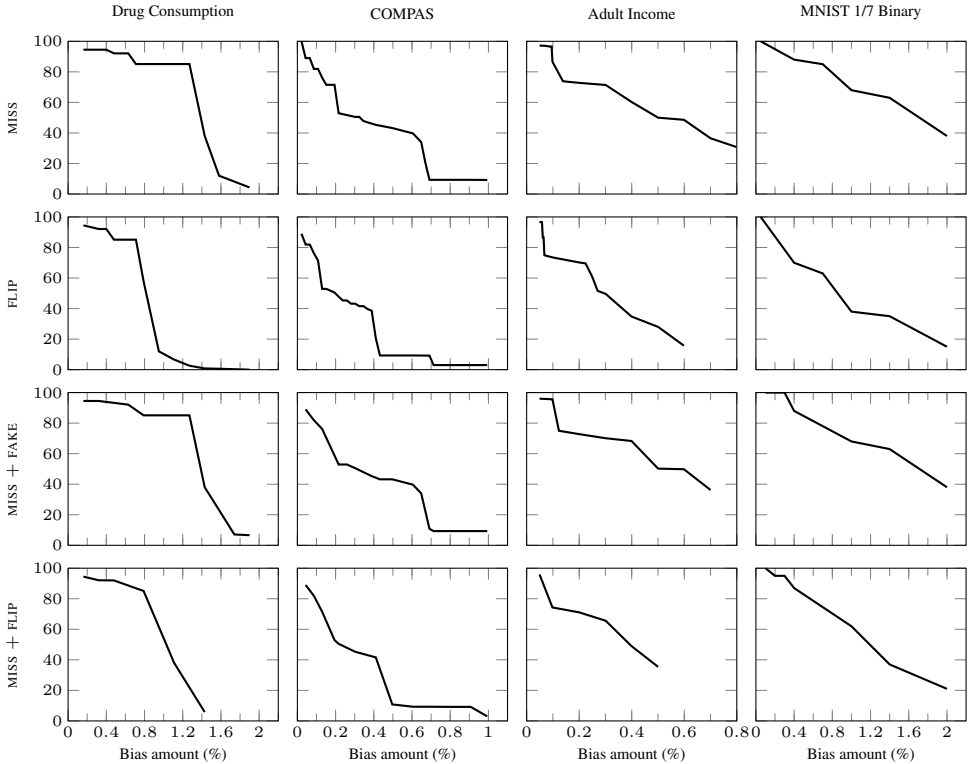

%% file: extra_demo_graphs.tex
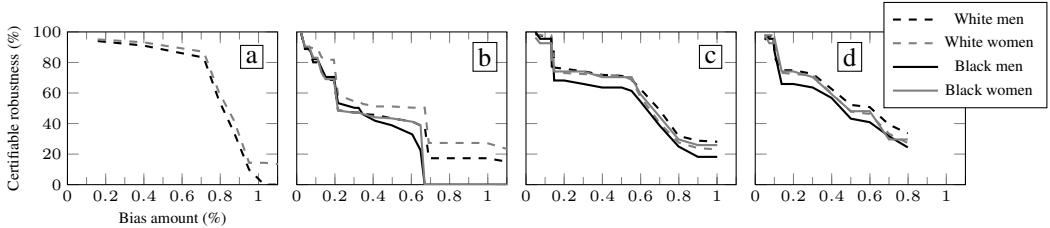
\begin{figure}[t]
\centering
\tiny
\pgfplotsset{filter discard warning=false}

\pgfplotscreateplotcyclelist{whatever}{%
    black,thick,dashed,every mark/.append style={fill=blue!80!black},mark=none\\%
    gray,thick,dashed,every mark/.append style={fill=red!80!black},mark=none\\%
    black,thick,every mark/.append style={fill=blue!80!black},mark=none\\%
    gray,thick,every mark/.append style={fill=red!80!black},mark=none\\%
    }
    
\begin{tikzpicture}
    \begin{groupplot}[
            group style={
                group size=4 by 1,
                horizontal sep=.1in,
                vertical sep=.05in,
                ylabels at=edge left,
                yticklabels at=edge left,
                xlabels at=edge bottom,
                xticklabels at=edge bottom,
            },
            height=.8in,
            xlabel near ticks,
            ylabel near ticks,
            scale only axis,
            width=0.2*\textwidth,
            xtick={0,.2,.4,.6,.8,1,1.2,1.4},
            minor xtick={0,0.1,0.2,0.3,0.4,0.5,0.6,0.7,0.8,0.9,1,1.1,1.2,1.3},
            xmin=0,
            xmax=1.1
        ]

        \nextgroupplot[
            ylabel=Certifiable robustness (\%),
            ymin=0,
            ymax=100,
            cycle list name=whatever,
            xlabel=Bias amount (\%)]
        \addplot table [x=bias_pct,y=male, col sep=comma]{data/drugs_demo.csv};
        \addplot table [x=bias_pct,y=female, col sep=comma]{data/drugs_demo.csv};

        \nextgroupplot[
            ymin=0,
            ymax=100,
            ytick={0,20,60,40,80},
             cycle list name=whatever,
        ]
        \addplot table [x=bias, y=whitemen, col sep=comma]{data/compas_missing_demo.csv};
        \addplot table [x=bias, y=whitewomen, col sep=comma]{data/compas_missing_demo.csv};
        \addplot table [x=bias, y=blackmen, col sep=comma]{data/compas_missing_demo.csv};
        \addplot table [x=bias, y=blackwomen, col sep=comma]{data/compas_missing_demo.csv};

        \nextgroupplot[
            ymin=0,
            ymax=100,
            xtick={0,.2,.4,.6,.8,1,1.2,1.4,1.6,1.8},
            minor xtick={0,0.1,0.2,0.3,0.4,0.5,0.6,0.7,0.8,0.9,1,1.1,1.2,1.3,1.4,1.5,1.6,1.7,1,8},
             cycle list name=whatever
        ]
        \addplot table [x=bias, y=whitemen, col sep=comma]{data/adult_labels_targeted_demo.csv};
        \addplot table [x=bias, y=whitewomen, col sep=comma]{data/adult_labels_targeted_demo.csv};
        \addplot table [x=bias, y=blackmen, col sep=comma]{data/adult_labels_targeted_demo.csv};
        \addplot table [x=bias, y=blackwomen, col sep=comma]{data/adult_labels_targeted_demo.csv};

        \nextgroupplot[
            ymin=0,
            ymax=100,
            ytick={0,20,40,60,80},
            legend style={at={(1,0.5)},anchor=south},
             cycle list name=whatever,
        ]
        \addplot table [x=bias, y=whitemen, col sep=comma]{data/adult_missing_demo.csv};
        \addplot table [x=bias, y=whitewomen, col sep=comma]{data/adult_missing_demo.csv};
        \addplot table [x=bias, y=blackmen, col sep=comma]{data/adult_missing_demo.csv};
        \addplot table [x=bias, y=blackwomen, col sep=comma]{data/adult_missing_demo.csv};

\legend{White men,White women,Black men,Black women};

    
    \end{groupplot}
    \node[draw] at (2.45,1.7) {\normalsize a};
    \node[draw] at (5.55,1.7) {\normalsize b};
    \node[draw] at (8.55,1.7) {\normalsize c};
    \node[draw] at (10.4,1.7) {\normalsize d};

\end{tikzpicture}
\caption{Left to right: Certifiable robustness by demographic group on (a) 
Drug Consumption  under $\flipB{}{}$ (NOTE: on this graph, the dark dotted line is all men, and the light dotted line is all women); (b) COMPAS under $\missB{}{}$; (c) Adult Income under $\flipB{}{\pred}$ where $\pred \triangleq (\emph{gender}=\textrm{Female} \land \emph{label}=\textrm{negative})$; (d) Adult Income under $\missB{}{}$ }\label{fig:extra_demo}
\end{figure}